\documentclass[aps,amsmath,pre,floatfix,superscriptaddress]{revtex4-2}
\usepackage{graphicx}
\usepackage{color}
\usepackage{geometry}
\usepackage{wrapfig}
\usepackage{comment}
\usepackage{amsmath,amsthm}
\usepackage[export]{adjustbox}
\usepackage{float}
\usepackage{amssymb}
\usepackage[utf8]{inputenc}
\usepackage{afterpage}
\usepackage{lipsum}  
\usepackage{appendix}
\usepackage{natbib}
\usepackage{subfigure}
\usepackage{subcaption}
\usepackage{amsmath}
\usepackage{setspace}
\newtheorem{theorem}{Theorem}[section]

\newtheorem{lemma}[theorem]{Lemma
}
\usepackage{tcolorbox}
\tcbuselibrary{skins}
\tcbuselibrary{breakable}

\usepackage{hyperref}
\hypersetup{colorlinks=true, linktoc=all, linkcolor=blue, linktocpage, citecolor=blue}

\usepackage{geometry}\usepackage{geometry}

\usepackage[normalem]{ulem}
\graphicspath{{figs/}}

\usepackage[section]{placeins}

\begin{document}
\title{Media and responsible AI governance: a game-theoretic and LLM analysis 
}




\author{Nataliya Balabanova$^{1}$}
\author{Adeela Bashir$^{2}$}
\author{Paolo Bova$^{2}$}
\author{Alessio Buscemi$^{3}$}
\author{Theodor Cimpeanu$^{4}$}
\author{Henrique Correia da Fonseca$^{5}$}
\author{Alessandro Di Stefano$^{2}$}
\author{Manh Hong Duong$^{1}$}
\author{Elias Fernández Domingos$^{6,7}$}
\author{Antonio Fernandes$^{5}$}
\author{The Anh Han$^{2,*}$}
\author{Marcus Krellner$^{4}$}
\author{Ndidi Bianca Ogbo$^{2}$}
\author{Simon T. Powers$^{8}$}
\author{Daniele Proverbio$^{9}$}
\author{Fernando P. Santos$^{10}$}
\author{Zia Ush Shamszaman$^{2}$}
\author{Zhao Song$^{2}$}



 \maketitle
	{\footnotesize
		\noindent
        $^{1}$ School of Mathematics, University of Birmingham\\
		$^{2}$ School Computing, Engineering and Digital Technologies, Teesside University\\
        $^{3}$ Luxembourg Institute of Science and Technology\\
		$^{4}$  School of Mathematics and Statistics, University of St Andrews\\
        $^{5}$ INESC-ID and Instituto Superior Técnico, Universidade de Lisboa \\
        $^{6}$ Machine Learning Group, Universit\'e libre de Bruxelles\\ 
        $^{7}$ AI Lab, Vrije Universiteit Brussel\\
	    $^{8}$ Division of Computing Science and Mathematics, University of Stirling\\        
        $^{9}$ Department of Industrial Engineering, University of Trento\\
        $^{10}$ University of Amsterdam \\

\noindent  $^\star$ Corresponding author: The Anh Han (T.Han@tees.ac.uk)
	}

  \maketitle

\section*{abstract}

\noindent This paper investigates the complex interplay between AI developers, regulators, users, and the media in fostering trustworthy AI systems.  Using evolutionary game theory and large language models (LLMs), we model the strategic interactions among these actors under different regulatory regimes.  The research explores two key mechanisms for achieving responsible governance, safe AI development and adoption of safe AI: incentivising effective regulation through media reporting, and conditioning user trust on commentariats' recommendation.  The findings highlight the crucial role of the media in providing information to users, potentially acting as a form of "soft" regulation by investigating developers or regulators, as a substitute to institutional AI regulation (which is still absent in many regions). Both game-theoretic analysis and LLM-based simulations reveal conditions under which effective regulation and trustworthy AI development emerge, emphasising the importance of considering the influence of different regulatory regimes from an evolutionary game-theoretic perspective. The study concludes that effective governance  requires managing incentives and costs for high quality commentaries.    

\vspace{3mm}

\noindent\textbf{Keywords:} AI governance, AI regulation, responsible AI, game theory, LLM, trustworthy AI, behavioural dynamics. 
\newpage
\section{Introduction}





A common narrative poses that the route to trustworthy artificial intelligence (AI) is enhanced though transparency and regulation of AI systems. In this account, regulation will incentivise developers to build trustworthy AI, which users are then justified in trusting and adopting. However, this interpretation ignores the complex socio-technical environment in which developers, regulators and users are embedded \citep{powers2023stuff}. Governments, and the regulators appointed by them, are both self-interested agents that can be expected to make strategic decisions \cite{clark2019regulatory, anderljung2023frontier,bengio2024managing,hammond2025multiagentrisksadvancedai,hadfield2023regulatory}. Likewise, developers are also self-interested actors whose goals may not completely align with the goals of governments, regulators, and ultimately users. Moreover, when users make a decision about whether to trust a particular AI system or not, they base this decision on a number of factors, including their prior dispositions, the quality of information about the system they have access to, and their trust in institutions such as scientists, regulators and the media \citep{lewis2022like,sutrop2019should,lansing2016trust,andras2018trusting}. In the process of ensuring trustworthy, beneficial, and trusted AI, accounting for these complex aspects is key to designing effective regulatory mechanisms.

Unfortunately, clear and abundant data about the effects of different regulatory mechanisms and strategies are not yet available; moreover, given the rapid pace of AI development, waiting for this data to become available before comparing possible mechanisms is not even affordable and desirable, as the landscape may have already changed. Evolutionary Game Theory (EGT) modelling \citep{sigmund2010calculus,hofbauer1998evolutionary}, grounded in widely accepted theories about how people and self-interested organisations behave \citep{binmore2005natural}, can provide a solution, by providing theoretical predictions about how people and organisations might behave under different conditions \citep{han2020regulate, han2022voluntary, bova2023both}. In previous work, we developed EGT models to compare the effectiveness of different mechanisms to incentivise independent regulators to monitor the behaviour of developers effectively. We found that effective regulation and safe development required users to condition their trust in developers on the effectiveness of regulators \citep{alalawi2024trust}. Nevertheless, this work left unanswered the question of \textit{how} users would obtain information on the behaviour of developers and regulators. In fact, there is increasing evidence that people's trust in AI developers is affected by media consumption \citep{yang2023ai}. To this aim, we here explore the role of media and other opinion leaders -- which we hereafter refer to as \textit{commentariat} -- in providing this information through investigative journalism \citep{maggetti_media_2012}.
Crucially, we consider the fact that the commentariat can themselves be self-interested agents, who do not merely report objectively on developments, but can also act to shape the agenda \citep{mccombs1972agenda}. 

In this work, we develop an EGT model to explore and quantify the role of the commentariat as self-interested agent in informing users' trust decisions. We consider two possible roles for the commentariat. First, they can choose to investigate developers, thereby potentially acting as a form of ``soft'' regulation on developers' behaviour. Second, they can act as a watchdog on the behaviour of regulators. We compare the effectiveness of these two distinct roles on incentivising developers to build safe AI systems, and for users to trust these systems. In our model, we investigate two potential strategies for the commentariat, either investing resources in providing quality information (cooperate), or spend less effort in investigations (defect). Users can condition their decision based on the information the commentariat provides, either trusting and adopting the AI system (conditional trust), or choosing not to trust and not adopt it. Developers can decide to invest time and effort in creating safe AI systems (cooperate), or avoid the burden of doing this (defect). 

We present a framework for formalising the strategic interaction between users, commentariat, AI system developers, and regulators as a game (Fig.~\ref{fig1}).

In addition to using traditional game theoretic approaches, recent research has suggested that Large Language Models (LLM) can be used to conduct experiments on EGT models \cite{zhao2023survey,lu2024llms}, and LLMs have been hypothesised to enable suitable replicas of human actions \cite{park2023generative, bail2024can}.
We thus complement our investigation by leveraging LLMs; to this end, we create a new framework to investigate the regulatory dynamics among the four agents (commentariat, developers, regulators, and users) considered in our model, and compare the results with game theoretic predictions. In our setting, four LLM agents interact dynamically, after being prompted in such a way as to represent the four desired actors. Because it has been observed that different LLMs may produce contrasting results in various tasks \cite{buscemi2024large, buscemi2024chatgpt, lee2024evaluating}, we employ two different models: ChatGPT-4o from OpenAI's GPT family \cite{chatgpt} and Mistral Large by Mistral \cite{mistral}.

\begin{figure*}
\begin{center}
\includegraphics[width=1\textwidth]{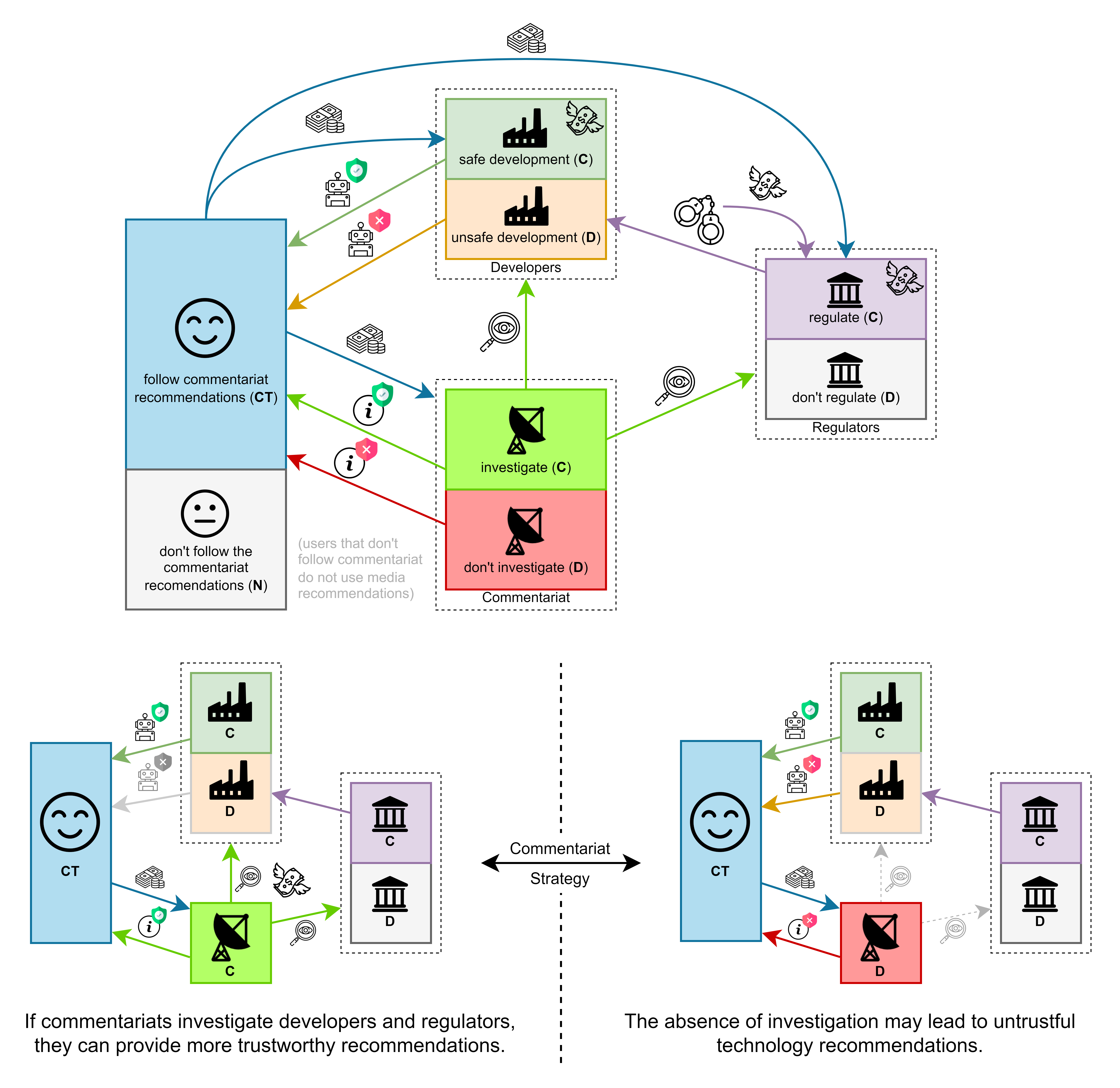}
\caption{\textbf{Core features}. The figure schematically illustrates the core features of the four-population model of AI governance. Users can either follow commentariat recommendations (Conditional Trust - CT) or not (N) (left). If they decide to follow the recommendations, their payoff depend on whether the commentariat invested in providing high quality information (investigate) or not (do not investigate) (lower middle).
Commentariat can investigate either developers or regulators.
Developers (upper middle) can either defect by creating unsafe AI products (D) or cooperate with regulations by creating safe ones (C), which entails additional costs. Regulators receive a benefit when users adopt AI systems. The regulator (right) can choose to cooperate by investing resources in regulating effectively (C) or not (D). If cooperating regulators catch unsafe developers, the latter are punished.}
\label{fig1}
\end{center}
\end{figure*}

By incorporating the commentariat as a distinct agent, along with users, developers and regulators, we aim to address the following key questions:
\begin{enumerate}
    \item What are the conditions under which quality investigation (cooperation) by the commentariat can foster responsible innovation, effective regulation of AI and users' trust? 
    \item Under which conditions will the commentariat carry out effective investigation of developers and regulators (cooperate)?
    \item Is it more effective for the commentariat to investigate and provide information on developers or regulators?
\end{enumerate}


In the next section, we describe the models and methods, including a four-population model of AI governance and the evolutionary methods for analysing the model from both finite and infinite population perspectives. Results for each type of analysis and Discussion sections will follow. 
\vspace{-0.3in}
\section{Models and Methods}


\subsection{Four population model of AI governance}

We start by constructing a model of an AI regulatory ecosystem, extending the three population model in \citep{alalawi2024trust} to capture the role of commentariat. 
The model involves four populations representing the four actors in the regulatory ecosystem: AI users, commentariat, developers, and regulators. In each population, individuals can choose different actions (also called strategies). 
In our model, a user can decide to follow commentariat recommendations (Conditional Trust -- CT) or not (N). If the user decides to follow the recommendations, the payoff will depend on whether the commentariat invests in providing high quality information (investigate) or not (do not investigate). Commentariat can investigate either developers or regulators. Developers can decide to defect by creating unsafe AI products (D) or cooperate by creating safe ones (C), which entails additional costs. Regulators receive a benefit when users adopt AI systems, for example through taxation on sales. The regulator can decide to invest in regulating effectively (C) at some cost, or not invest in regulating effectively (D). If cooperating regulators catch
unsafe developers, the latter are punished (see Table~\ref{tab:role_explanation}). Table~\ref{tab:parameters} explains the key parameters of the models.

\begin{table}[!h]
    \centering
\begin{tabular}{c|p{2cm}|p{7cm}}
   \textbf{Role}  & \textbf{Actions} & \textbf{Explanation} \\
   \hline
   Commentators  & \centering C/D & Investigates and provides an \emph{informed} recommendation (C), which means that it makes transparent the action of the developer/regulator, or provides an \emph{uninformed} recommendation (D) \\
   \hline
   Users & \centering CT/N & Either follows the commentator recommendations about whether to adopt the technology (CT) or never adopts the technology (N) \\
   \hline
   Developers & \centering C/D & Produces a SAFE (C) or UNSAFE (D) technology \\
   \hline
   Regulators & \centering C/D & The regulator can decide to invest in regulating effectively (C) by paying the cost, or do not regulate effectively (D). 
\end{tabular}

    \caption{\textbf{Roles and their possible actions in the AI regulatory ecosystem.}}
    \label{tab:role_explanation}
\end{table}

\begin{table}[!h]
    \centering
    \begin{tabular}{c|p{10cm}}
        \textbf{Parameter}  & \textbf{Explanation} \\
        \hline
        $b_{I}$  & Reputational benefit a commentator receives when making a correct recommendation \\
        \hline
        $b_{U}$ & Benefit a user receives when adopting a safe technology \\
        \hline
        $b_{P}$ & Benefit a developer receives when their technology is adopted \\
        \hline
        $b_{R}$ & Benefit a regulator receives when a user adopts the technology \\
        \hline
        $b_fo$ & Benefit a regulator receives when catching unsafe behaviour from a developer \\
        \hline
        $c_{I}$ & Cost for a commentator of providing an informed recommendation \\
        \hline
        $c_w$ & Reputational cost to a commentator of making an incorrect recommendation \\
        \hline
        $\epsilon$ & Fraction of user benefit when developers play \textit{D}, where $\varepsilon$ in [-$\infty$,1], also referred to as the (inverse) risk factor users take when adopting the technology \\
        \hline
        $c_{P}$ & Additional cost of creating safe AI (the cost of creating unsafe AI is normalised to 0) \\
        \hline
        $u$ & Cost of being punished (for a developer for being found developing unsafely) \\
        \hline
        $v$ & Cost for a regulator for punishing unsafe developers \\
        \hline
        $c_{R}$ & The cost of effective regulation (the cost of not doing this is normalised to 0) \\
        \hline
        $p_{w}$ & Probability that the recommendation of a commentator is \emph{incorrect} when they defect 
    \end{tabular}

    \caption{\textbf{Explanation of the key parameters of the models.}}
    \label{tab:parameters}
\end{table}

\begin{figure*}
\begin{center}
\includegraphics[width=0.7\textwidth]{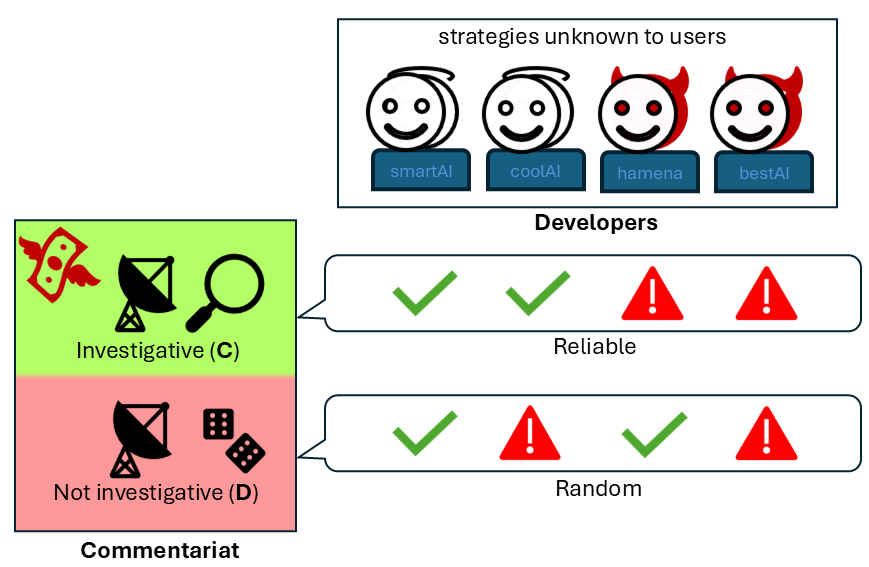}
\caption{\textbf{Function of the Commentariat}. The figure schematically illustrates what kind of information the two types of commentariat provides. The investigative (cooperating) will pay a cost to look for the hidden strategies of the developers, while the not investigative (defecting) will shun the cost and misclassify a developer with probability $p_w$.}
\label{fig2}
\end{center}
\end{figure*}

\begin{table}[p]
 \caption{\textbf{Payoff matrix} for the AI Governance model where \textbf{commentators investigate developers} (Commentariat (\textit{Com}), User, developer (\textit{Dev}), and Regulator (\textit{Reg})).}
    \centering
    \begin{tabular}{c|c|c|c||p{3cm}|p{2cm}|p{2.5cm}|p{3.5cm}}
        \multicolumn{4}{c||}{Actions} & \multicolumn{4}{c}{Payoffs} \\
        \hline
        Com & User & Dev & Reg & Com & User & Dev & Reg \\
        \hline
        \hline
        C & CT & C & C & $b_{I} - c_{I}$ & $b_{U}$ & $b_{P} - c_{P}$ & $b_{R} - c_{R}$ \\
        \hline
        C & CT & C & D & $b_{I} - c_{I}$ & $b_{U}$ & $b_{P} - c_{P}$ & $b_{R}$ \\
        \hline
        C & CT & D & C & $b_{I} - c_{I}$ & $0$ & $0$ & $- c_{R}$ \\
        \hline
        C & CT & D & D & $b_{I} - c_{I}$ & $0$ & $0$ & $0$ \\
        \hline
        C & N & C & C & $-c_{I}$ & $0$ & $-c_{P}$ & $-c_{R}$ \\
        \hline
        C & N & C & D & $-c_{I}$ & $0$ & $-c_{P}$ & $0$ \\
        \hline
        C & N & D & C & $-c_{I}$ & $0$ & $0$ & $-c_{R}$ \\
        \hline
        C & N & D & D & $-c_{I}$ & $0$ & $0$ & $0$ \\
        \hline
        D & CT & C & C & $(1-p_{w})b_{I} - p_{w}c_{w}$ & $(1-p_{w})b_{U}$ & $(1-p_{w})b_{P} - c_{P}$ & $(1-p_{w})b_{R} - c_{R}$ \\
        \hline
        D & CT & C & D & $(1-p_{w})b_{I} - p_{w}c_{w}$ & $(1-p_{w})b_{U}$ & $(1-p_{w})b_{P} - c_{P}$ & $(1-p_{w})b_{R}$ \\
        \hline
        D & CT & D & C & $(1-p_{w})b_{I} - p_{w}c_{w}$ & $p_{w}\epsilon b_{U}$ & $p_{w}(b_{P} - u)$ & $p_{w}(b_{R} + b_{fo} - v) - c_{R}$ \\
        \hline
        D & CT & D & D & $(1-p_{w})b_{I} - p_{w}c_{w}$ & $p_{w}\epsilon b_{U}$ & $p_{w}b_{P}$ & $p_{w}b_{R}$ \\
        \hline
        D & N & C & C & $0$ & $0$ & $-c_{P}$ & $-c_{R}$ \\
        \hline
        D & N & C & D & $0$ & $0$ & $-c_{P}$ & $0$ \\
        \hline
        D & N & D & C & $0$ & $0$ & $0$ & $-c_{R}$ \\
        \hline
        D & N & D & D & $0$ & $0$ & $0$ & $0$ \\
        \hline
        
    \end{tabular}
    \label{tab:payoff_matrix_1}
\end{table}

The individual payoff earned in any one encounter (also called a game) depends on the strategy of the participating individuals. In each game, one user, one developer, one commentariat and one regulator participate. If the user follows commentariat recommendations when the commentariat is investing in providing an informed recommendation, and both the developer and the regulator cooperate (by complying and enforcing, respectively), the user benefits significantly from AI adoption, denoted by $b_U$. On the other hand, if the developer defects by not complying with the regulations, the user adopting AI is affected by unsafe AI, gaining a reduced or even negative benefit, denoted by $\epsilon \times b_U$, where $\epsilon \in [-\infty, 1]$. This parameter, $\epsilon$, also represents a \textit{risk factor} that users take when trusting and adopting the AI system.

Developers receive a benefit, denoted by $b_P$, when their technology is adopted, e.g., through sales. Complying with the regulations carries an additional cost $c_P$ of creating safe AI. While if they do not comply with the regulations and develop AI unsafely, they may be punished at the cost $u$ if they are found to be defecting. 

Regulators also earn a benefit, denoted by $b_R$, when the user trusts and adopts the technology. This corresponds to regulation being funded by taxes on the sales of AI products, or by governments investing more in regulation when there is more uptake of AI. Regulators pay a cost, denoted by $c_R$, to carry out effective regulation, e.g., through thorough auditing. When they pay this cost, we assume that they are rewarded an amount of $b_{fo}$ when they catch unsafe developers' behaviour (when users trust and adopt AI). In this case, the cooperative regulator pays an additional cost $v$ to administer this punishment.

Commentators receive a reputational benefit denoted by $b_I$, when they provide a correct recommendation about the safety of the AI system. They can pay a cost $c_I$ to provide an informed recommendation, which ensures that the recommendation is correct. On the other hands, defecting commentators do not pay this cost, but they can still earn the benefit $b_I$ if the recommendation happens to be correct with a probability $p_w$. If they defect, and make the wrong recommendation with probability $1-p_w$, they suffer a reputational cost of $c_w$.

As illustrated in Fig.~\ref{fig1}, commentators can decide to investigate or not either the developers or the regulators (see green arrows from the commentariat to developers and regulators). Fig.~\ref{fig2} schematically explains the role of the commentariat, which depends on the type of information they provide. 


Tables ~\ref{tab:payoff_matrix_1} and ~\ref{tab:payoff_matrix_2} define the payoffs in the two cases of commentariat investigating developers or regulators, respectively.

\begin{table}[h!]
    \centering
    \normalsize
    \caption{\textbf{Payoff matrix} for the AI Governance model where \textbf{commentators investigate AI regulators} (Commentariat (\textit{Com}), User, developer (\textit{Dev}), and Regulator (\textit{Reg})).}
    \centering
        \begin{tabular}{c|c|c|c||p{3cm}|p{2cm}|p{2.5cm}|p{5cm}}
        \multicolumn{4}{c||}{Actions} & \multicolumn{4}{c}{Payoffs} \\
        \hline
        Com & User & Dev & Reg & Com & User & Dev & Reg \\
        \hline
        \hline
        C & CT & C & C & $b_{I} - c_{I}$ & $b_{U}$ & $b_{P} - c_{P}$ & $b_{R} - c_{R}$ \\
        \hline
        C & CT & C & D & $b_{I} - c_{I}$ & $0$ & $ - c_{P}$ & $0$ \\
        \hline
        C & CT & D & C & $b_{I} - c_{I}$ & $\epsilon b_U$ & $b_P - u$ & $b_R - c_R - v  + b_{fo}$ \\
        \hline
        C & CT & D & D & $b_{I} - c_{I}$ & $0$ & $0$ & $0$ \\
        \hline
        C & N & C & C & $-c_{I}$ & $0$ & $-c_{P}$ & $-c_{R}$ \\
        \hline
        C & N & C & D & $-c_{I}$ & $0$ & $-c_{P}$ & $0$ \\
        \hline
        C & N & D & C & $-c_{I}$ & $0$ & $0$ & $-c_{R}$ \\
        \hline
        C & N & D & D & $-c_{I}$ & $0$ & $0$ & $0$ \\
        \hline
        D & CT & C & C & $(1-p_{w})b_{I} - p_{w}c_{w}$ & $(1-p_{w})b_{U}$ & $(1-p_{w})b_{P} - c_{P}$ & $(1-p_{w})b_{R} - c_{R}$ \\
        \hline
        D & CT & C & D & $(1-p_{w})b_{I} - p_{w}c_{w}$ & $p_{w}b_{U}$ & $p_{w}b_{P} - c_{P}$ & $p_{w}b_{R}$ \\
        \hline
        D & CT & D & C & $(1-p_{w})b_{I} - p_{w}c_{w}$ & $(1-p_{w})\epsilon b_{U}$ & (1-$p_{w})(b_{P} - u)$ & $(b_{R} -  c_{R}+ b_{fo} - v)(1-p_{w}) $ -  $p_W c_R$ \\
        \hline
        D & CT & D & D & $(1-p_{w})b_{I} - p_{w}c_{w}$ & $p_{w}\epsilon b_{U}$ & $p_{w}b_{P}$ & $p_{w}b_{R}$ \\
        \hline
        D & N & C & C & $0$ & $0$ & $-c_{P}$ & $-c_{R}$ \\
        \hline
        D & N & C & D & $0$ & $0$ & $-c_{P}$ & $0$ \\
        \hline
        D & N & D & C & $0$ & $0$ & $0$ & $-c_{R}$ \\
        \hline
        D & N & D & D & $0$ & $0$ & $0$ & $0$ \\
        \hline
        
    \end{tabular}
    \label{tab:payoff_matrix_2}
\end{table}

\subsection{AI agents setup}
\label{sub:AI_ag_setup}

The games are set using LLM agents whose payoffs are given as described above. To setup agents within a game-theoretic framework, we employ the Framework for AI Agents Bias Recognition using Game Theory (FAIRGAME) \cite{buscemi2025fairgame}. 
FAIRGAME enables testing of user-defined games, described in textual format and incorporating any desired payoff matrix.
Additionally, it allows for the specification of agent traits that will participate in these games.
The agents can be instantiated using any LLM of choice by invoking the corresponding APIs.

To run, FAIRGAME requires the following inputs:
\begin{itemize}
    \item \textbf{Configuration File:} A file that defines the setup of both the agents and the game. The default format is JSON.
    \item \textbf{Prompt Template:} A text file that defines the instruction template, providing a literal description of the game. It includes placeholders that are dynamically populated with information from the configuration file at each round, ensuring customization for each agent.
\end{itemize}

\begin{table}[h!]
    \centering
    \normalsize
    \caption{Parameters provided to FAIRGAME.}
    \label{tab:fairgame}

    \begin{tabular}{p{0.5\textwidth}|p{0.3\textwidth}}
        \hline
        \textbf{Parameter} &  \textbf{Value} \\
        \hline
            \textbf{Number of agents} & 4 \\
        \hline
            Names of the agents & regulator; developer; user; commentariat \\
        \hline
            Personalities of the agents & None; None; None; None\\
         \hline
            Underlying LLM & OpenAI GPT-4o; Mistral Large \\
        \hline
            Number of rounds & 1; \\
        \hline
            Agents communicate & False \\
         \hline
            Agents know the personalities of the others & False \\    
        \hline
            Stopping condition & None \\    
        \hline
        \end{tabular}
\end{table}

Table~\ref{tab:fairgame} presents the parameters used in the experiments, as specified in the configuration file. FAIRGAME simulates interactions among four distinct agents, each fulfilling a designated role: regulator, developer, user, and commentariat.

As reported in the table, the LLM underlying these agents is either OpenAI's GPT-4o or Mistral Large. Each simulation maintains consistency in model selection across all agents, meaning that in some experiments, all agents operate using GPT-4o, while in others, they all rely on Mistral Large. No experiment combines different models within a single game.

The study focuses on one-shot games, each consisting of a single round. Agents make decisions autonomously, without interacting with one another, ensuring complete independence in their actions. Furthermore, they are unaware of the personalities or strategic inclinations of their counterparts. While the framework allows for defining agent personalities, the main experiments set all personalities to None. This ensures that decisions are guided purely by their assigned roles, reflecting the default behavior of the LLMs without external influences.

Lastly, the game runs for the specified number of rounds without a predetermined stopping condition.
The template used for all experiments is available in Appendix \ref{app:template}.

\subsection{Methods}
\subsubsection{Stochastic dynamics for finite populations}
\label{subsection:finitepopulation}
\paragraph{Payoff calculation.} We consider four different well-mixed populations of Commentariat
 (Co), Users (U), developers (C) and Regulators (R) of sizes, respectively $N_{Co}$, $N_U$, $N_C$ and $N_R$. 
Let $x$ be the fraction of commentariats  that cooperate. Let $y$, $z$ and $\omega$ be respectively the fraction of  users that trust the AI system, and developers and Regulators that cooperate. Each game involves an individual randomly drawn from each population. The fitness that a commentariat, user, developer and regulator obtains in each game is respectively given by:
\begin{align}
f^{Co}_{X\in\{C,D\}}&=yzw P^{Co}_{XTCC}+yz(1-w) P^{Co}_{XTCD}+y(1-z)w P^{Co}_{XTDC}+y(1-z)(1-w) P^{Co}_{XTDD} 
\notag\\&\quad
+(1-y)zwP^{Co}_{XNCC}+(1-y)z(1-w) P^{Co}_{XNCD}+(1-y)
\notag\\&\quad 
+(1-y)(1-z)(1-w) P^{Co}_{XNDD}, \label{eq: UTN}  \\
f^U_{Y\in\{T,N\}}& = xzw P^{U}_{CYCC} + xz(1-w) P^{U}_{CYCD}   + x(1-z)w P^{U}_{CYDC} + x(1-z)(1-w) P^{U}_{CYDD} 
\notag\\&\quad
+   (1-x)zw P^{U}_{DYCC} + (1-x)z(1-w) P^{U}_{DYCD}   + (1-x)(1-z)w P^{U}_{DYDC} 
\notag\\&\quad
+ (1-x)(1-z)(1-ww) P^{U}_{CYDD}\\
  f^C_{Z\in\{C,D\}}& = xyw P^{U}_{CTZC} + xy(1-w) P^{U}_{CTZD}   + x(1-y)w P^{U}_{CNZC} + x(1-y)(1-w) P^{U}_{CNZD} 
\notag\\&\quad
+ (1-)yw P^{U}_{DTZC} + (1-x)y(1-w) P^{U}_{DTZD}   + (1-x)(1-y)w P^{U}_{DNZC} 
\notag\\&\quad
+ (1-x)(1-y)(1-w) P^{U}_{DNZD} \\
f^C_{W\in\{C,D\}}& = xyz P^{U}_{CTCW} +  xy(1-z) P^{U}_{CTDW} + x(1-y)z P^{U}_{CNCW} + x(1-y)(1-z) P^{U}_{CNDW}
\notag\\&\quad
+(1-x)yz P^{U}_{DTCW} +  (1-x)y(1-z) P^{U}_{DTDW} + (1-x)(1-y)z P^{U}_{DNCW}
\notag\\&\quad+ (1-x)(1-y)(1-z) P^{U}_{DNDW} 
\end{align}

The fitness (i.e. average payoff) is computed using the payoff matrix constructed in the models (see Tables~\ref{tab:payoff_matrix_1}-\ref{tab:payoff_matrix_2}). 

\vspace{3mm}
\paragraph{Evolutionary dynamics.} For a finite population setting, at each time step, a randomly selected individual A, with fitness $f_A$, may adopt a different strategy by imitating a randomly chosen individual B from the same population (with fitness $f_B$) with probability given by the Fermi distribution \cite{traulsen2006}.
\[
p=[1+e^{-\beta(f_B-f_A)}]^{-1},
\]
where $\beta\geq 0$ is the strength of selection. $\beta = 0$ corresponds to neutral drift where imitation decisions are random, while for large $\beta \rightarrow \infty$, the imitation decision becomes increasingly deterministic.

In the absence of mutations or exploration, the end states of evolution are inevitably monomorphic: once such a state is reached, it cannot be escaped through imitation. We thus further assume that with a certain mutation probability,  an agent switches randomly to a different strategy without imitating another agent.  In the limit of small mutation rates, the dynamics will proceed with, at most, two strategies in the population, such that the behavioural dynamics can be conveniently described by a Markov chain, where each state represents a monomorphic population, whereas the transition probabilities are given by the fixation probability of a single mutant \citep{key:imhof2005,key:novaknature2004,domingos2023egttools}. The resulting Markov chain has a stationary distribution, which characterises the average time the population spends in each of these monomorphic end states.


Now, the probability to change the number $k$ of agents using strategy A by $\pm$ one in each time step can be written as ($Z$ is the population size) \citep{traulsen2006}:
\begin{equation} 
T^{\pm}(k) = \frac{Z-k}{Z} \frac{k}{Z} \left[1 + e^{\mp\beta[f_A(k) - f_B(k)]}\right]^{-1}.
\end{equation}
The fixation probability of a single mutant with a strategy A in a population of $(Z-1)$ agents using B is given by \citep{traulsen2006,key:novaknature2004}:
\begin{equation} 
\label{eq:fixprob} 
\rho_{B,A} = \left(1 + \sum_{i = 1}^{Z-1} \prod_{j = 1}^i \frac{T^-(j)}{T^+(j)}\right)^{-1}.
\end{equation} 

The transition matrix $\Lambda$ corresponding to the set of $\left\{1,\ldots ,s\right\}$ strategies is given by:
\begin{eqnarray}\label{eq:2.6}
\Lambda_{ij,j\neq i}=\frac{\rho_{ji}}{4}\hspace{1mm} \text{ and } \hspace{1mm} \Lambda_{ii}= 1- \sum_{j=1,j\neq i}^s \Lambda_{ij}.
\end{eqnarray}
  Fixation probability $\rho_{ij}$ denotes the likelihood that a population transitions from a state $i$ to a different state $j$ when a mutant of one of the populations adopts an alternate strategy \textit{s}. The fixation probability is divided by the number of populations (3) representing the interaction of three players at a time \cite{encarnaccao2016paradigm,alalawi2019pathways}. 

\subsubsection{Population dynamics for infinite populations: The multi-population replicator dynamics}
\label{subsection:infinitepopulation}
In this section, we recall the framework of the replicator dynamics for multi-populations \cite{taylor1979evolutionarily,bauer2019stabilization}. To describe the dynamics, we consider a set of $m$ different populations ($m$ is some positive integer), which are infinitely large and well-mixed. Each population $i$, $i=1,\ldots m$, consists of $n_i$ ($n_i$ is some positive integer) different strategies (types). Let $x_{ij}, 1\leq i\leq m, 1\leq j\leq n_i$, be the frequency of the strategy $j$ in the population $i$. We denote by $x_i=(x_{ij})_{j=1}^{n_i}$, which is the collection of all strategies in the population $i$, and $x=(x_1,\ldots, x_m)$, which is the collection of all strategies in all populations. 

For each $i\in\{1,\ldots, m\}$ and $j\in\{1,\ldots, n_i\}$, let $f_{ij}(x)$ be the fitness (reproductive rate) of the strategy $j$ in the population $i$. This fitness is obtained when the strategy $j$ interacts with all other strategies in all populations; thus, it depends on all the strategies in the populations. The average fitness of the population $i$ is defined by
\[
\bar{f}_i(x)=\sum_{j=1}^{n_i} x_{ij} f_{ij}(x).
\]
The multi-population replicator dynamics is then given by
\begin{equation}
\label{eq: general replicator dynamics}
\dot{x}_{ij}=x_{ij} (f_{ij}(x)-\bar{f}_{i}(x)), \quad 1\leq i\leq m,\quad 1\leq j\leq n_i. 
\end{equation}
This is in general an ODE system of $\sum_{i=1}^m n_i$ equations. Noting, however that since $\sum_{j=1}^{n_i}x_{ij}=1$ for all $i=1,\ldots, m$, we can reduce the above system to a system of $\sum_{i=1}^m n_i-m$ equations. 

Now we focus on the case when there are two strategies in each population (which is the case for our models of AI governance and trust in the present paper), that is $n_i=2$ for all $i=1,\ldots, n_i$. Let $\eta_i$ be the frequency of the first strategy in the population $i$, $i=1,\ldots, m$ (thus $1-\eta_i$ will be the frequency of the second strategy in the population $i$), let $\eta=(\eta_1,\ldots, \eta_m)$. Let $f_{1i}(\eta)$ and $f_{2i}(\eta)$ be the fitness of the first and second strategy in the population $i$. Since:
\[
\bar{f}_i(\eta)=\eta_i f_{1i}(\eta)+(1-\eta_i) f_{2i}(\eta),
\]
we have:
\[
f_{1i}(\eta)-\bar{f}_i(\eta)=f_{1i}(\eta)-(\eta_i f_{1i}(\eta)+(1-\eta_i) f_{2i}(\eta))=(1-\eta_i)(f_{1i}(\eta)-f_{2i}(\eta)).
\]
Thus we obtain the following system of equations:
\begin{equation}
\label{eq: general replicator2}
\dot{\eta}_i=\eta_i(1-\eta_i)(f_{1i}(\eta)-f_{2i}(\eta)), \quad i=1,\ldots, m.
\end{equation}

This is a system of $m$ coupled nonlinear Ordinary Differential Equations (ODE) for $m$ variables.

In the subsequent sections, we employ \eqref{eq: general replicator2} to our models of AI governance trust, where the fitnesses are computed from the payoff matrix constructed in the models, see Tables~\ref{tab:payoff_matrix_1}-\ref{tab:payoff_matrix_2}. The resulting replicator dynamics for both models, (see equations \eqref{eq: replicator dynamics baseline model} and equations \eqref{eq: replicator dynamics model 2} below) can be written in a general form as:
\begin{subequations}
\label{eq: general system}
\begin{align}
 &\dot{x}=x(1-x) F_1(X)=:\tilde{F}_1(X),\\
&\dot{y}=y(1-y) F_2(X)=:\tilde{F}_2(X),\\
&\dot{z}=z(1-z) F_3(X)=:\tilde{F}_3(X),\\
&\dot{w}=w(1-w) F_4(X)=:\tilde{F}_4(X), 
\end{align}
\end{subequations}
where $X=(x,y,z,w)$ is the vector of frequencies, and $F_i(X)$ ($i=1,\ldots, 4$) are the corresponding difference of the fitnesses in the two models (precise formulas are given in the next section).

The 16 vertices $(x,y,z,w)\in\{0,1\}^4$ of the 4-dimensional cube are obviously equilibria of \eqref{eq: general system} (called vertical equilibria). An internal equilibria $X$ is a solution in $(0,1)^4$ of the following system of equations:
\begin{equation*}
 F_1(X)=F_2(X)=F_3(X)=F_4(X)=0.   
\end{equation*}
Analytically computing these internal equilibria and analysing their stable properties are very complicated due to the nonlinearity and number of the parameters, thus we will do so numerically. Here, we analyze the stability of the vertical equilibria. Let $X^*=(x^*,y^*,z^*, w^*)\in \{0,1\}^4$ be a vertical equilibrium. Since $x^*\in \{0,1\}$, we have:
\begin{align*}
&\frac{\partial}{\partial x} \tilde{F}_1(X)\Big\vert_{X=X^*}=(1-x^*)F_1(X^*)-x^* F_1(X^*)+x^*(1-x^*) \partial_x F_1(X^*)=(1-2x^*)F_1(X^*),\\
&\frac{\partial}{\partial t} \tilde{F}_1(X)\Big\vert_{X=X^*}=x^*(1-x^*) \partial_t F_1(X^*)=0.
\end{align*}
Similar computations yield:
\begin{align*}
 &\frac{\partial}{\partial y} \tilde{F}_2(X)\Big\vert_{X=X^*}=(1-2y^*)F_2(X^*),\\
&\frac{\partial}{\partial t} \tilde{F}_2(X)\Big\vert_{X=X^*}=0,\quad t\in\{x,z,w\},\\
 &\frac{\partial}{\partial z} \tilde{F}_3(X)\Big\vert_{X=X^*}=(1-2z^*)F_3(X^*),\\
&\frac{\partial}{\partial t} \tilde{F}_3(X)\Big\vert_{X=X^*}=0,\quad t\in\{x,y,w\},\\
 &\frac{\partial}{\partial w} \tilde{F}_4(X)\Big\vert_{X=X^*}=(1-2w^*)F_4(X^*),\\
&\frac{\partial}{\partial t} \tilde{F}_4(X)\Big\vert_{X=X^*}=0,\quad t\in\{x,y,z\}.
\end{align*}
From the above calculations, the Jacobian matrix $J(X)=\frac{D \tilde{F}(X)}{D X}$, evaluated at $X^*$, $J(X^*)=\frac{D \tilde{F}(X)}{D X}\vert_{X=X^*}$, will be the following diagonal matrix:
\[
J(X^*)=\mathrm{diag}((1-2x^*)F_1(X^*), (1-2y^*)F_2(X^*),(1-2z^*)F_3(X^*),(1-2w^*)F_4(X^*)).
\]
The diagonal entries are the real eigenvalues of $J(X^*)$. Thus, $X^*$ is stable if and only if the following conditions hold:
\[
(1-2x^*)F_1(X^*)<0,~ (1-2y^*)F_2(X^*)<0,~ (1-2z^*)F_3(X^*)<0,~(1-2w^*)F_4(X^*)<0.
\]
From these conditions, one can easily determine the stability of $X^*$ which depends on the parameters. We will employ this analysis for our models in the next section.
\section{Equilibrium analysis in infinite populations}

\subsubsection{Model I: developers are investigated by media } 
In order to write the replicator dynamics explicitly using the general framework \eqref{eq: general replicator2}, we need to derive the fitness differences for all participants in the game. We use the values from Table \ref{tab:payoff_matrix_1}; for brevity, we omit the explicit calculations. For commentariat and users, the following hold:
\begin{align}
    f^{Co}_C-f^{Co}_D &= \left(y b_i-c_i\right) -\left(y \left(b_i-p_W \left(b_i+c_W\right)\right)\right)\notag\\&= y p_W \left(b_i+c_W\right)-c_i,
\end{align}
\begin{align}
f^U_T-f^U_N&= b_U \left((x-1) p_W ((z-1) \epsilon +z)+z\right) - 0\notag\\&=b_U \left((x-1) p_W ((z-1) \epsilon +z)+z\right).\label{eq: differnce T and N}
\end{align}
Similarly, the difference of the fitness between two strategies in creators is:
\begin{align}
f^C_C-f^C_{D}&= \left(y b_P \left((x-1) p_W+1\right)-c_P\right) - \left((x-1) y p_W \left(u w-b_P\right)\right)\notag\\&=(x-1) y p_W \left(2 b_P-u w\right)+y b_P-c_P \label{eq: difference C and D for developers}
\end{align}
Finally, the difference of the fitness between two strategies in regulators is:
\begin{align}
f^R_{C}-f^R_{D}&= \left(-(x-1) y p_W \left((z-1) \left(v-b_{\text{f0}}\right)+(1-2 z) b_R\right)+y z b_R-c_R\right) \notag\\&- \left(y b_R \left((x-1) (2 z-1) p_W+z\right)\right)\notag\\&=-(x-1) y (z-1) p_W \left(v-b_{\text{f0}}\right)-c_R.
\end{align}

The replicator dynamics immediately becomes:
\begin{subequations}
\label{eq: replicator dynamics baseline model}
\begin{align}
&\dot{x}=x(1-x)\Big[  y p_w \left(b_I+c_W\right)-c_I
\Big],\label{eq:repdyn-a}\\ 
&\dot{y}=y(1-y)\Big[  b_U \left((x-1) p_w ((z-1) \epsilon +z)+z\right)\Big],\\ 
&\dot{z}=z(1-z)\Big[ (x-1) y p_W \left(2 b_P-u w\right)+y b_P-c_P\Big],\\
&\dot{w}=z(1-z)\Big[  -(x-1) y (z-1) p_W \left(v-b_{\text{f0}}\right)-c_R\Big],\\
&(x(0),y(0),z(0),w(0))=(x_0,y_0,z_0,w_0).\label{eq:repdyn-e}
\end{align}
\end{subequations}

First, we investigate the existence and the number of equilibria in the $[0,1]^4$ hypercube of the above system. Clearly, $(x,y,z,w)\in \{0,1\}^4$ will be (vertical) equilibrium points. The full list of  equilibria with one of the variables lying on the boundary consists of 29 isolated non-degenerate equilibrium points and two edges:
 \begin{equation}
 \begin{split}
 &\begin{cases}
     x = 1,\\
     z=0,\\
     w=0.
 \end{cases}, \ \  \begin{cases}
     x = 1,\\
     z=0,\\
     w=1.
 \end{cases}
 \end{split}
 \end{equation}
The two possible candidates for internal equilibria are: 
\begin{small}
    \begin{equation}
   \begin{cases}
       x=  \frac{-\sqrt{-2 (\epsilon -1) c_i c_R p_W \left(v-b_{\text{f0}}\right) \left(b_i+c_W\right)+c_i^2 \left(v-b_{\text{f0}}\right){}^2+(\epsilon +1)^2 c_R^2 p_W^2 \left(b_i+c_W\right){}^2}+c_i \left(2 p_W-1\right) \left(v-b_{\text{f0}}\right)+(\epsilon +1) c_R p_W \left(b_i+c_W\right)}{2 c_i p_W \left(v-b_{\text{f0}}\right)}\\
       y= \frac{c_i}{p_W \left(b_i+c_W\right)}\\
       z= \frac{\sqrt{-2 (\epsilon -1) c_i c_R p_W \left(v-b_{\text{f0}}\right) \left(b_i+c_W\right)+c_i^2 \left(v-b_{\text{f0}}\right){}^2+(\epsilon +1)^2 c_R^2 p_W^2 \left(b_i+c_W\right){}^2}+c_i \left(v-b_{\text{f0}}\right)+(\epsilon +1) c_R p_W \left(b_i+c_W\right)}{2 c_i \left(v-b_{\text{f0}}\right)}\\
       
       w= -\frac{b_P \left(\sqrt{-2 (\epsilon -1) c_i c_R p_W \left(v-b_{\text{f0}}\right) \left(b_i+c_W\right)+c_i^2 \left(v-b_{\text{f0}}\right){}^2+(\epsilon +1)^2 c_R^2 p_W^2 \left(b_i+c_W\right){}^2}+c_i \left(b_{\text{f0}}-v\right)\right)}{2 u c_Rp_W \left(b_i+c_W\right)}\\ \ \ \ +\frac{c_P \left(\sqrt{-2 (\epsilon -1) c_i c_R p_W \left(v-b_{\text{f0}}\right) \left(b_i+c_W\right)+c_i^2 \left(v-b_{\text{f0}}\right){}^2+(\epsilon +1)^2 c_R^2 p_W^2 \left(b_i+c_W\right){}^2}+c_i \left(b_{\text{f0}}-v\right)+(\epsilon +1) c_R p_W \left(b_i+c_W\right)\right)}{c_i 2 u c_R}\\
       \ \ \ +\frac{(\epsilon -3) b_P c_R}{2 u c_R}
   \end{cases}
    \end{equation}
\end{small}
and
\begin{small}
    \begin{equation}
        \begin{cases}
          x =   \frac{\sqrt{-2 (\epsilon -1) c_i c_R p_W \left(v-b_{\text{f0}}\right) \left(b_i+c_W\right)+c_i^2 \left(v-b_{\text{f0}}\right){}^2+(\epsilon +1)^2 c_R^2 p_W^2 \left(b_i+c_W\right){}^2}+c_i \left(2 p_W-1\right) \left(v-b_{\text{f0}}\right)+(\epsilon +1) c_R p_W \left(b_i+c_W\right)}{2 c_i p_W \left(v-b_{\text{f0}}\right)},\\
          
          y= \frac{c_i}{p_W \left(b_i+c_W\right)},\\
          z= \frac{-\sqrt{-2 (\epsilon -1) c_i c_R p_W \left(v-b_{\text{f0}}\right) \left(b_i+c_W\right)+c_i^2 \left(v-b_{\text{f0}}\right){}^2+(\epsilon +1)^2 c_R^2 p_W^2 \left(b_i+c_W\right){}^2}+c_i \left(v-b_{\text{f0}}\right)+(\epsilon +1) c_R p_W \left(b_i+c_W\right)}{2 c_i \left(v-b_{\text{f0}}\right)},\\
          
          w =  \frac{b_P \left(\sqrt{-2 (\epsilon -1) c_i c_R p_W \left(v-b_{\text{f0}}\right) \left(b_i+c_W\right)+c_i^2 \left(v-b_{\text{f0}}\right){}^2+(\epsilon +1)^2 c_R^2 p_W^2 \left(b_i+c_W\right){}^2}+c_i \left(v-b_{\text{f0}}\right)\right)}{2 u c_Rp_W \left(b_i+c_W\right)}\\ \ \ \ +\frac{c_P \left(-\sqrt{-2 (\epsilon -1) c_i c_R p_W \left(v-b_{\text{f0}}\right) \left(b_i+c_W\right)+c_i^2 \left(v-b_{\text{f0}}\right){}^2+(\epsilon +1)^2 c_R^2 p_W^2 \left(b_i+c_W\right){}^2}+c_i \left(b_{\text{f0}}-v\right)+(\epsilon +1) c_R p_W \left(b_i+c_W\right)\right)}{2 u c_Rc_i}\\ \ \ \ +\frac{(3-\epsilon ) b_P c_R}{2 u c_R}.
        \end{cases}
    \end{equation}
\end{small}
Analytically determining whether they are indeed internal equilibria, that is, whether they lie inside the hypercube $(0,1)^4$ is intractable. We thus invoke numerical analysis. Figures \ref{fig:barcharta} and \ref{fig:barchartb} show the distribution of the number of solutions for randomly chosen values of the parameters, for both models I and II. Nevertheless, we obtain the following interesting result that provides simple conditions on the non-existence of internal equilibria.
\begin{lemma}
 There are no internal equilibria in $[0,1]^4$ when either of the following hold:
 \begin{enumerate}
     \item $v-b_{f_0}>0$,
     \item $0<\epsilon<1$. 
 \end{enumerate}
\end{lemma}
\begin{proof}

Internal equilibria are given by the solutions of the last parts of the equations \ref{eq:repdyn-a}-\ref{eq:repdyn-e}. 

From the first equation it is clear that $y= \frac{c_i}{p_W \left(b_i+c_W\right)}$. Substituting this value into the fourth equation yields:
\begin{equation*}
    -\frac{(x-1) (z-1) c_i \left(v-b_{\text{f0}}\right)}{b_i+c_W}-c_R=0.
\end{equation*}

Since $c_i,c_R>0$ and $0<x<1$, $0<z<1$, if $v-b_{f_0}>0$, the whole expression turns strictly negative and can't result in an equilibrium.   

Analogously, substituting the internal equilibrium value of $y$ into the second equation and solving it for $z$ gives:
\begin{equation*}
    z =-\frac{(1-x) \epsilon  p_W}{1-(1-x) (\epsilon +1) p_W},
\end{equation*}
which is negative when $0<\epsilon<1$. 
\end{proof}

\subsubsection{Model I -- Stability analysis }

Due to the form of the equations, as it has been proved in section \ref{subsection:infinitepopulation}, the Jacobian matrix (as described there) will be diagonal at the vertices of the hypercube. Therefore, the stability of a vertical equilibrium will be determined by the values on the diagonal of the matrix, which are shown in Table \ref{tab:payoff_matrix_2}. Recall that the points with four positive nonzero eigenvalues are unstable, four nonzero negative are stable; the remaining are saddles.

Additionally, note the presence of degenerate equilibrium points in the table -- this is to be expected since equilibria occupy edges of the cube. 
The following hold:
\begin{itemize}
    \item $X^{\ast} = (0,0,0,0)$ is stable if and only if $\epsilon<0$;
    \item $X^{\ast} = (0,1,0,0)$ is stable if and only if $p_w(b_i  + c_w)-c_i<0$, $-2b_Pp_W + b_p-c_p<0$ and $-p_W(v-b_{f_0})-c_R<0$;
    \item $X^{\ast} = (0,1,0,1)$  can be made either stable or a saddle or unstable by regulating the values of the parameters; this is the only vertex point with this property;
    \item $X^{\ast} = (1,1,0,1)$ is stable if and only is $c_i - p_w(b_i+c_W)<0$, $c_p-b_P<0$.
    
\end{itemize}

\begin{table}[ht]
\label{table:table6}
    \centering
        \caption{Vertex fixed points and eigenvalues Model I}
    \begin{tabular}{c|c|c|c|c}
    \hline
        $X^*$ & $\lambda_1$ & $\lambda_2$ & $\lambda_3$ & $\lambda_4$ \\
        \toprule
        (0,0,0,0) & $-c_i$&$\epsilon  b_U p_W$&$-c_P$&$-c_R$\\
        \hline
         (0,0,0,1)& $-c_i$ &$\epsilon  b_U p_W$ &$-c_P$ &$c_R$\\
          \hline 
          (0,0,1,0) &$-c_i$ &$b_U \left(1-p_W\right)$ &$b_{f_o}+b_R-c_R-v - b_{f_o}p_w - 2 b_R p_w + v p_w$ &$-c_R$\\
          \hline
         (0,1,0,0) &$p_W \left(b_i+c_W\right)-c_i$ &$-\epsilon  b_U p_W$ &$-2 b_P p_W+b_P-c_P$ &$-p_W \left(v-b_{\text{f0}}\right)-c_R$\\
          \hline  
          (1,0,0,0) &$c_i$ &$0$ &$-c_P$ &$-c_R$\\
          \hline
         (1,1,0,0) &$c_i-p_W \left(b_i+c_W\right)$ &$0$ &$b_P-c_P$ &$-c_R$\\
          \hline 
         (1,0,1,0) &$c_i$ &$b_U$ &$c_P$ &$-c_R$\\
          \hline 
         (1,0,0,1) &$c_i$ &0&$-c_P$ &$c_R$ \\
          \hline 
         (0,1,0,1)& $p_W \left(b_i+c_W\right)-c_i$&$-\epsilon  b_U p_W$&$-p_W \left(2 b_P-u\right)+b_P-c_P$&$p_W \left(v-b_{\text{f0}}\right)+c_R$\\
  \hline
         (0,1,1,0) &$ p_W \left(b_i+c_W\right)-c_i$&$-b_U \left(1-p_W\right)$&$2 b_P p_W-b_P+c_P$&$-c_R$\\
          \hline 
         (0,0,1,1) &$ -c_i$&$b_U \left(1-p_W\right)$&$c_P$&$c_R$\\
          \hline
         (1,1,1,0)&$c_i-p_W \left(b_i+c_W\right)$&$-b_U$&$c_P-b_P$&$-c_R$\\ 
          \hline
         (1,1,0,1) &$c_i-p_W \left(b_i+c_W\right)$&0&$,b_P-c_P$&$c_R$\\
          \hline 
         (1,0,1,1) &$c_i$&$b_U$&$c_P$&$c_R$\\
          \hline
         (0,1,1,1) &$ p_W \left(b_i+c_W\right)-c_i$&$-b_U \left(1-p_W\right)$&$p_W \left(2 b_P-u\right)-b_P+c_P$&$c_R$\\
          \hline 
         (1,1,1,1) &$c_i-p_W \left(b_i+c_W\right)$&$-b_U$&$c_P-b_P$&$c_R$\\
         \hline
    \end{tabular}
    \label{tab:VerticalEquilibria_ModelI}
\end{table}
\subsubsection{Model II:  regulators  are investigated by media}
The replicator dynamics for Model II is constructed completely analogously to the previous case, based on the values from the Table \ref{tab:payoff_matrix_2}. The equations now read:
\begin{subequations}
\label{eq: replicator dynamics model 2}
\begin{align}
&\dot{x}=x(1-x)(f^{Co}_C-f^{Co}_D)=x(1-x)\left[-c_I + y(b_I + c_w)p_w\right],\\ 
&\dot{y}=y(1-y)(f^U_T-f^U_N)=y(1-y)\left[-b_U(\epsilon(-1 + z) - z)(w + (-1 + 2w)(-1 + x) p_w)\right],\\ 
&\dot{z}=z(1-z)(f^C_C-f^C_D)=z(1-z)\left[-c_P + uwy + uw(-1 + x)yp_wy\right],\\
&\dot{w}=z(1-z)(f^R_C-f^R_D)=z(1-z)\bigl[-c_R + y(b_{f_o} + b_R + v(-1 + z) - b_{f_o}z)\\&\ \ \  + (-1 + x)y(b_{f_o} + 2b_R + v(-1 + z) - b_{f_o}z)p_w\bigr],\\
&(x(0),y(0),z(0),w(0))=(x_0,y_0,z_0,w_0),
\end{align}
\end{subequations}
where $(x_0,y_0,z_0,w_0)\in [0,1]^4$ is the initial data.

Solving analytically gives 27 isolated equilibria and again two edges of degenerate equilibria:
\begin{equation}
    \begin{cases}
        x = 1,\\
        z = 0,\\
        w = 0
    \end{cases}, \ \begin{cases}
        x = 1,\\
        z = 1,\\
        w = 0,
    \end{cases}
\end{equation}
together with 2 possible internal equilibria. These equilbira sometimes lie in the hypercube; however, the probability of that occurring is fairly low (see Figure \ref{fig:barchartb}).

\subsubsection{Model II: - Stability analysis} 

\begin{table}[ht]
    \centering
        \caption{Vertex fixed points and eigenvalues Model II}
    \begin{tabular}{c|c|c|c|c}
    \hline
        $X^*$ & $\lambda_1$ & $\lambda_2$ & $\lambda_3$ & $\lambda_4$ \\
        \toprule
        (0,0,0,0) & $-c_I$ &$-c_P$ &$-c_R$ &$b_U\epsilon p_w$\\
        \hline
         (0,0,0,1)& $-c_I$ &$-c_P$ &$c_R$ &$b_U\epsilon (1-p_w)$\\
          \hline
         (0,0,1,0) &$-c_I$ &$c_P$ &$-c_R$ &$b_Up_w$\\
          \hline
         (0,1,0,0) &$-c_P$ &$b_U\epsilon p_w$ &$b_{f_o}+b_R-c_R-v - b_{f_o}p_w - 2 b_R p_w + v p_w$ &$-c_I + b_I p_w + c_w p_w$\\
          \hline
         (1,0,0,0) &0 &$c_I$ &$-c_P$ &$-c_R$\\
          \hline
         (1,1,0,0) &0 &$-c_P$ &$b_{f_o}+b_R-c_R-v$ &$c_I-b_Ip_w-c_wp_w$\\
          \hline
         (1,0,1,0) &0 &$c_I$ &$c_P$ &$-c_R$\\
          \hline
         (1,0,0,1) &$c_I$ &$-c_P$ &$c_R$ &$b_U\epsilon$\\
          \hline
         (0,1,0,1) &$-b_U\epsilon (1 - p_w)$ & $-c_P + u - 
 u p_w$ &$ -b_{f_o} - b_R + c_R + v + b_{f_o} p_w + 
 2 b_R p_w - v p_w$ &$ -c_I + b_Ip_w + c_w p_w$\\
  \hline
         (0,1,1,0) &$c_P$ &$-b_Up_w$ &$b_R-c_R-2b_Rp_w$ &$-c_I+b_Ip_w+c_wp_w$\\
          \hline
         (0,0,1,1) &$-c_I$ &$c_P$ &$c_R$ &$-b_U(1-p_w)$\\
          \hline
         (1,1,1,0) &$0$ &$c_P$ & $b_R - c_R$ & $c_I - b_Ip_w - c_w p_w$\\
          \hline
         (1,1,0,1) &$-b_U\epsilon$ &$-c_P+u$ &$-b_{f_o} - b_R + c_R + v$ & $c_I - b_I p_w - c_w p_w$\\
          \hline
         (1,0,1,1) &$b_U$ & $c_I$ &$c_P$ &$c_R$\\
          \hline
         (0,1,1,1) &$b_U (-1 + p_w)$ &$-b_R + c_R + 2 b_R p_w$ &$c_P - u + u p_w$ & $-c_I + b_I p_w + c_w p_w$\\
          \hline
         (1,1,1,1) &$-b_U$ & $-b_R+c_R$ &$ c_P-u$ &$c_I - b_I p_w - c_w p_w$\\
         \hline
    \end{tabular}
    \label{tab:VerticalEquilibria_ModelII}
\end{table}
We now determine the stability of a vertical equilibria $X^*\in\{0,1\}^4$ of the second model. Based on the eigenvalues of $X^*$, table \ref{tab:VerticalEquilibria_ModelII}, we obtain the following stable equilibria and the conditions:
\begin{itemize}
    \item $X^*=(0,0,0,0)$ is stable if and only if $\epsilon<0$.
    \item $X^*=(0,1,0,0)$ is stable if and only if:
    \[\epsilon>0, ~ b_{f_o}+b_R-c_R-v - b_{f_o}p_w - 2 b_R p_w + v p_w<0, ~ -c_I+b_Ip_w+c_wp_w<0.\]
    \item $X^*=(0,1,0,1)$ is stable if and only if:
    \[\epsilon>0,~ -c_P+u(1-p_w)<0, ~  -b_{f_o} - b_R + c_R + v + b_{f_o} p_w + 2 b_R p_w - v p_w<0, ~-c_I + b_Ip_w + c_w p_w<0. \]
     \item $X^*=(1,1,0,1)$ is stable if and only if:
     \[\epsilon>0, ~ -c_P+u<0, ~ -b_{f_o} - b_R + c_R + v<0,~c_I - b_I p_w - c_w p_w<0.\]
    \item $X^*=(1,1,1,1)$ is stable if and only if:
    \[-b_R+c_R<0,~c_P-u<0,~c_I-p_w(b_I+c_w)<0.\]
\end{itemize}
\subsubsection{Numerical results}
In this section, we numerically solve the replicator dynamics for models I and II, which are sets of ordinary differential equations given in \eqref{eq: replicator dynamics baseline model} and \eqref{eq: replicator dynamics model 2} respectively. The graphs of the solutions are represented in Figures \ref{fig:NumericalIntegration1} and \ref{fig:NumericalIntegration2}. 

For the chosen values of the parameters the equations  converge to a vertical equilibrium (one may observe additionally that the convergence is faster for Model I in Figure \ref{fig:NumericalIntegration1} and, in Figure \ref{fig:NumericalIntegration2}, faster for Model II). 

The behaviour in terms of convergence is similar for the two models as well, barring the cases with $c_I = 0.5, b_I = 1, c_R=5$ and $c_I = 0.5, b_I = 5, c_R=5$.



\begin{figure}[h]
    \centering
    \subfigure{
        \includegraphics[width=0.35\linewidth]{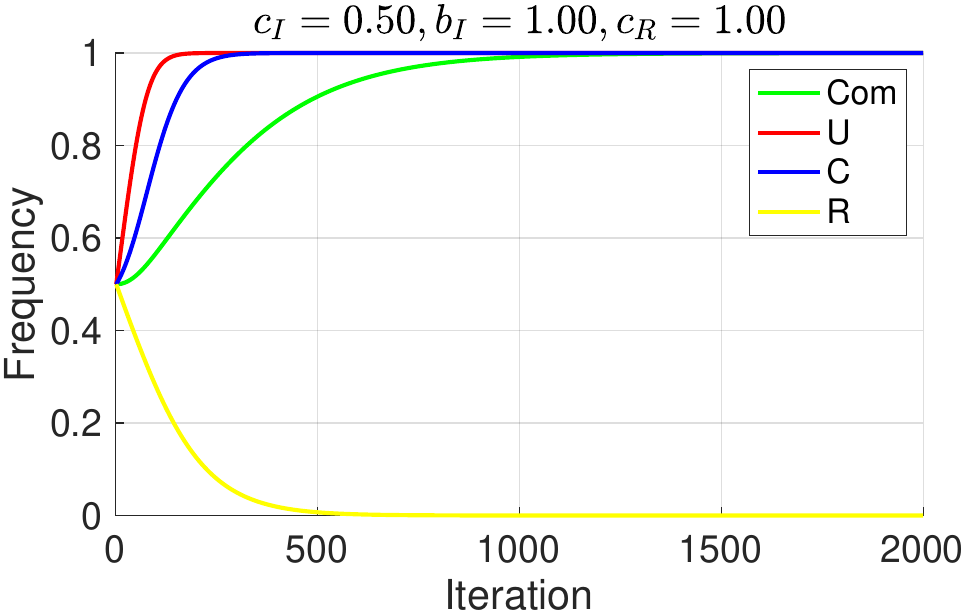} 
    }
    \subfigure{
        \includegraphics[width=0.35\linewidth]{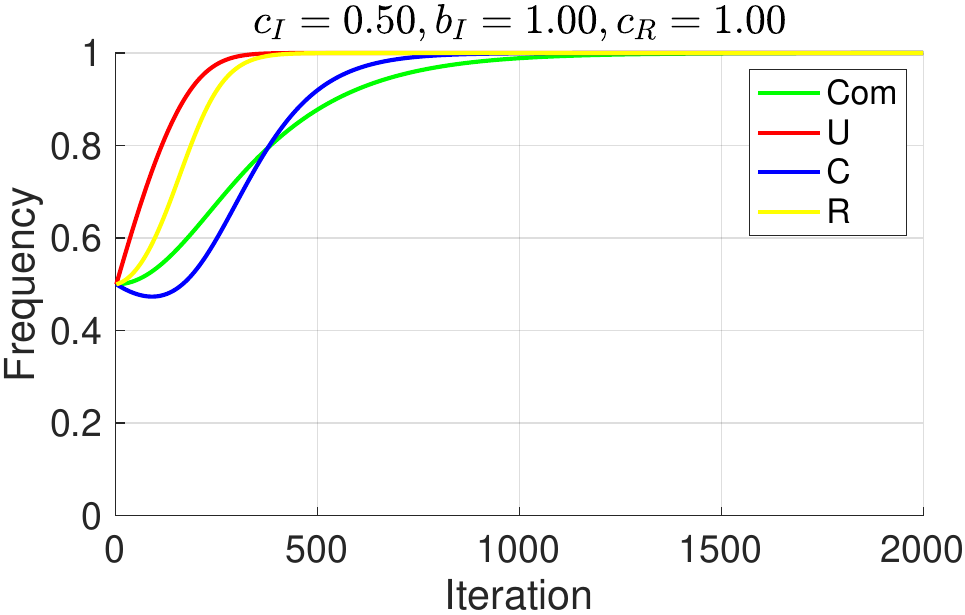} 
    }
    \subfigure{
        \includegraphics[width=0.35\linewidth]{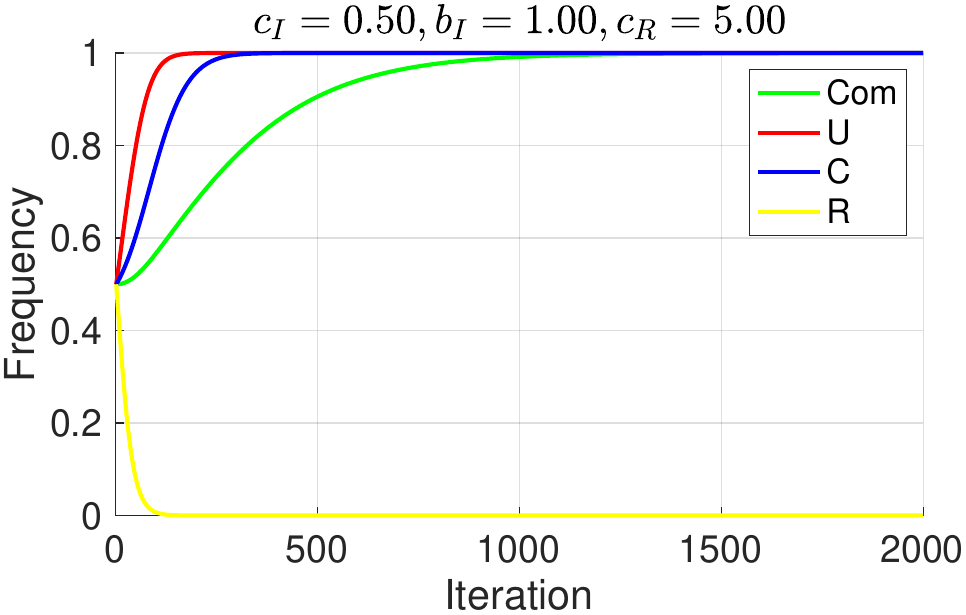}
    }
    \subfigure{
        \includegraphics[width=0.35\linewidth]{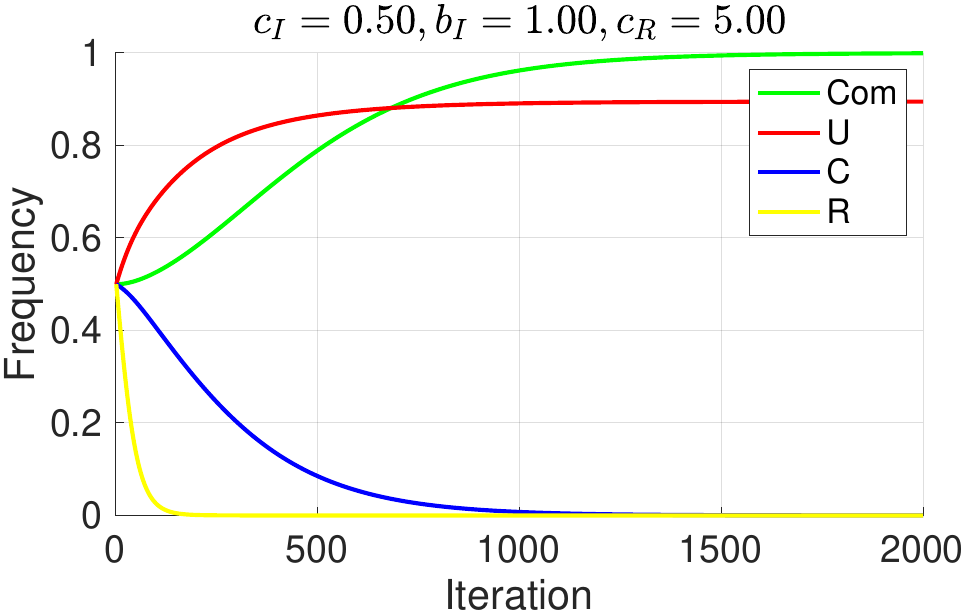}
    }
     \subfigure{
        \includegraphics[width=0.35\linewidth]{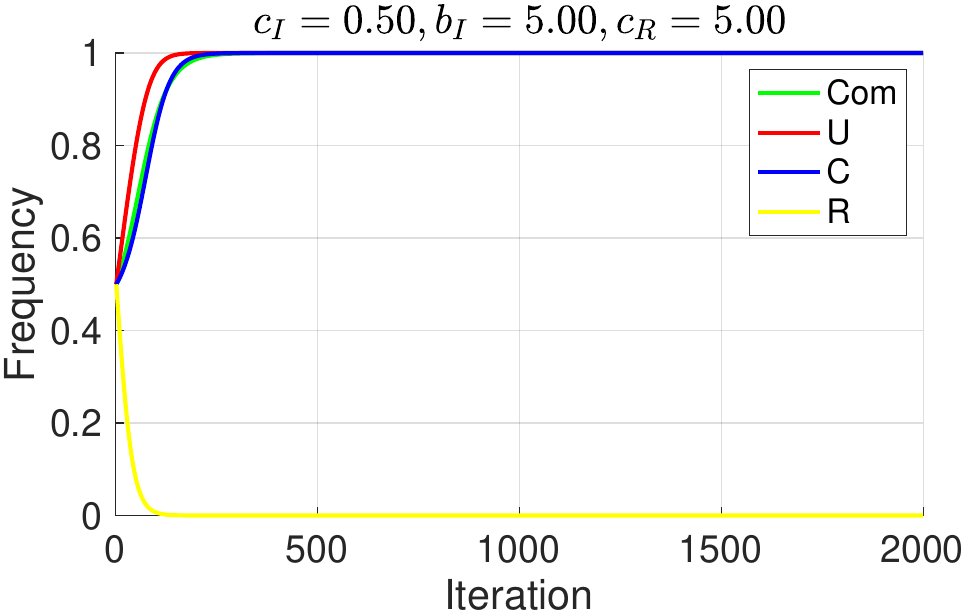}
    }
    \subfigure{
        \includegraphics[width=0.35\linewidth]{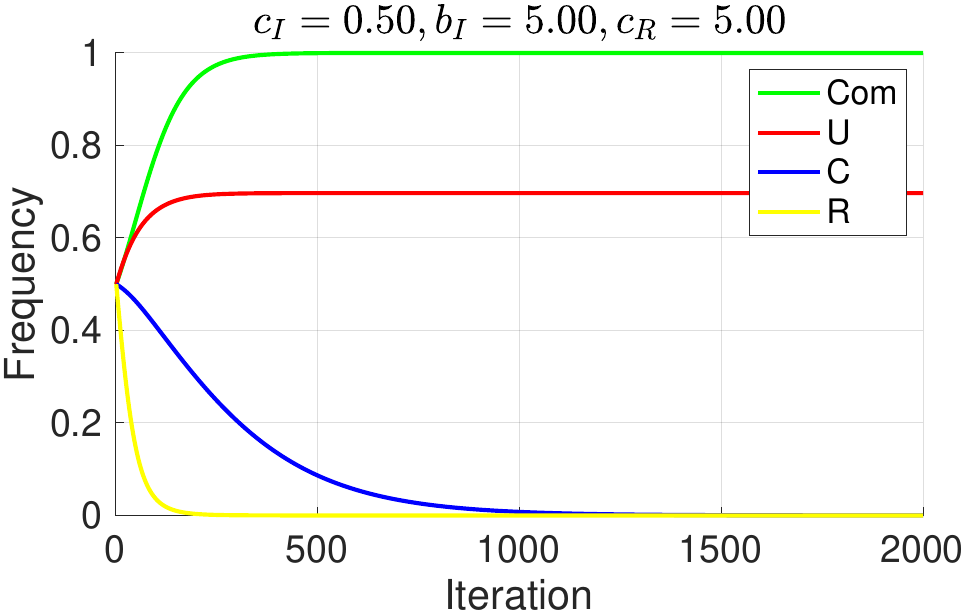}
    }
     \subfigure{
        \includegraphics[width=0.35\linewidth]{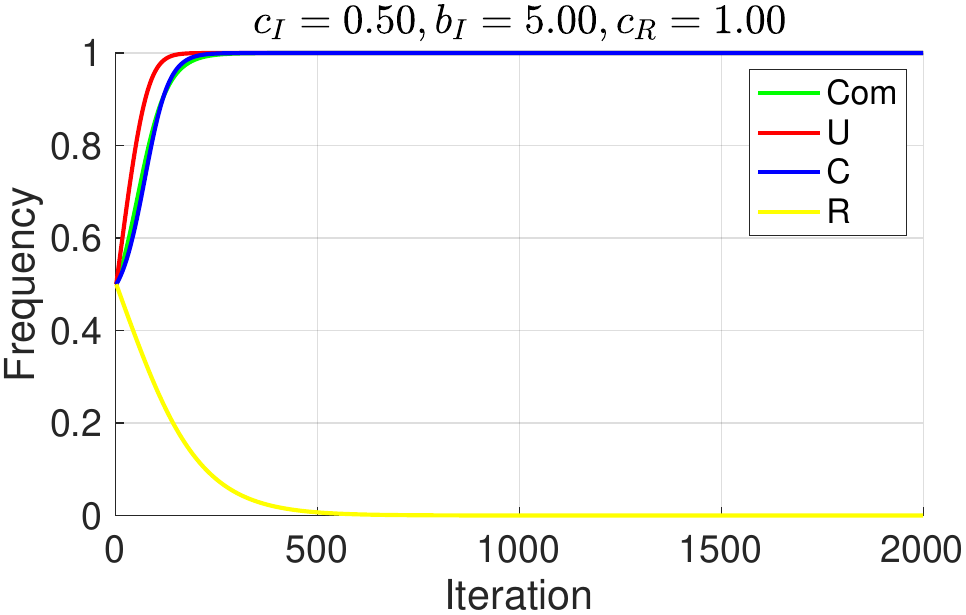}
    }
    \subfigure{
        \includegraphics[width=0.35\linewidth]{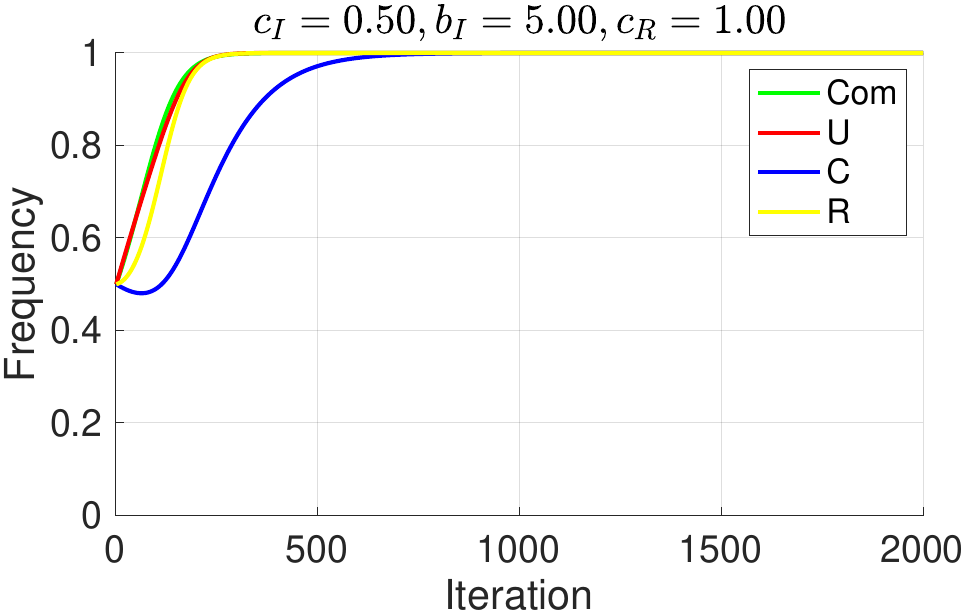}
    }
    \caption{Numerical integration of the evolution equation for two models \textbf{when the media investigation cost is low ($c_I=0.5$)}. The left column shows the results of Model I, and the right column shows the results of Model II. Parameters are set as $b_U = 4, b_P = 4, b_R = 4, c_P = 0.5, c_w = 1, u = 1.5, v = 0.5, b_{f_o} = 1, \epsilon = 0.2, p_w = 0.5$.}
    \label{fig:NumericalIntegration1}
\end{figure}

\begin{figure}[h]
    \centering
    \subfigure{
        \includegraphics[width=0.35\linewidth]{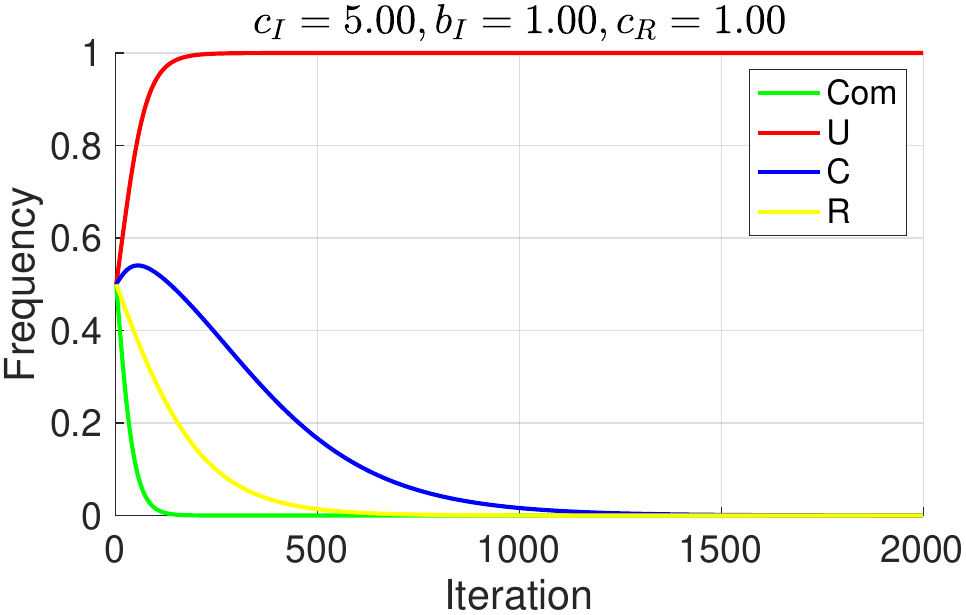} 
    }
    \subfigure{
        \includegraphics[width=0.35\linewidth]{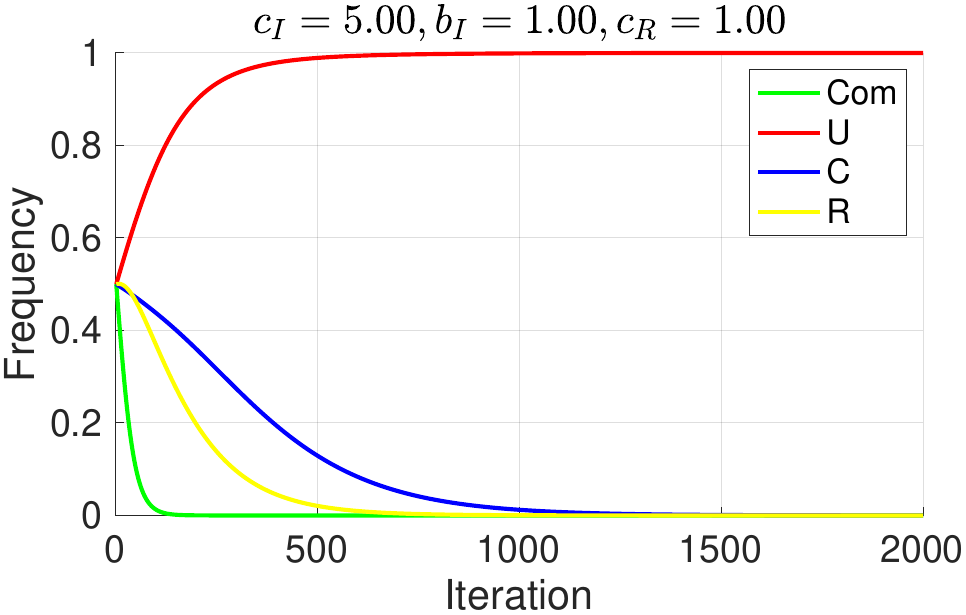} 
    }
    \subfigure{
        \includegraphics[width=0.35\linewidth]{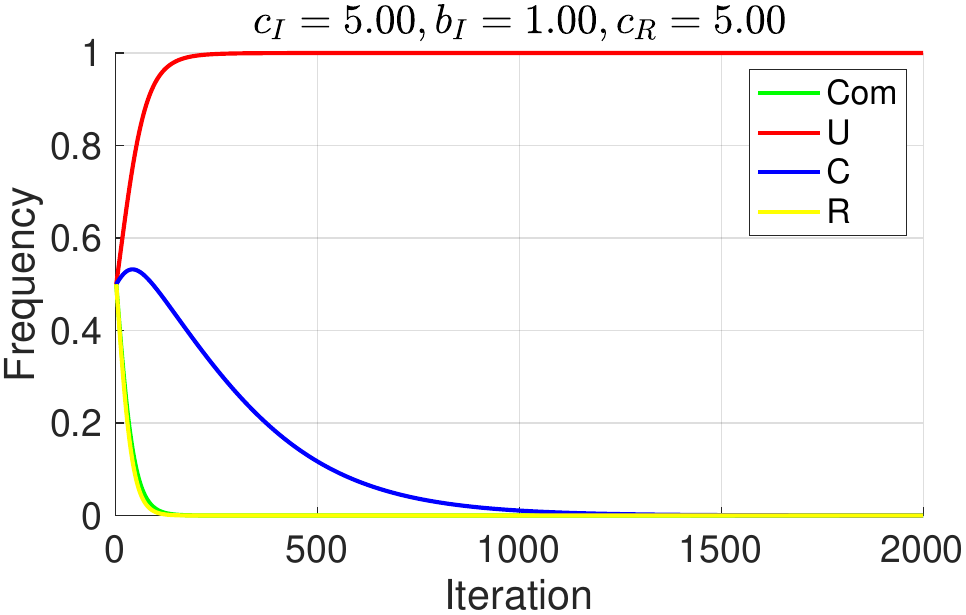}
    }
    \subfigure{
        \includegraphics[width=0.35\linewidth]{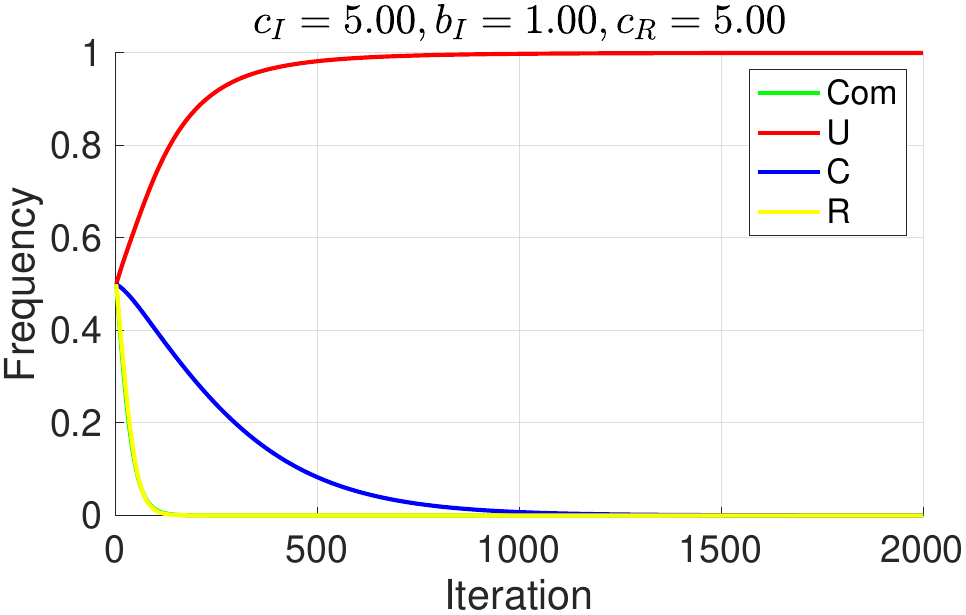}
    }
     \subfigure{
        \includegraphics[width=0.35\linewidth]{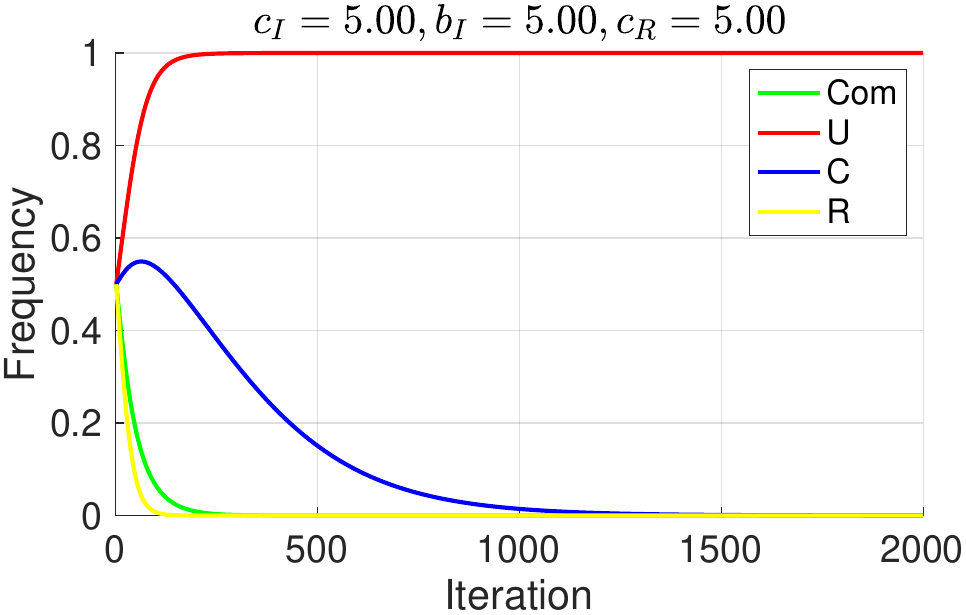}
    }
    \subfigure{
        \includegraphics[width=0.35\linewidth]{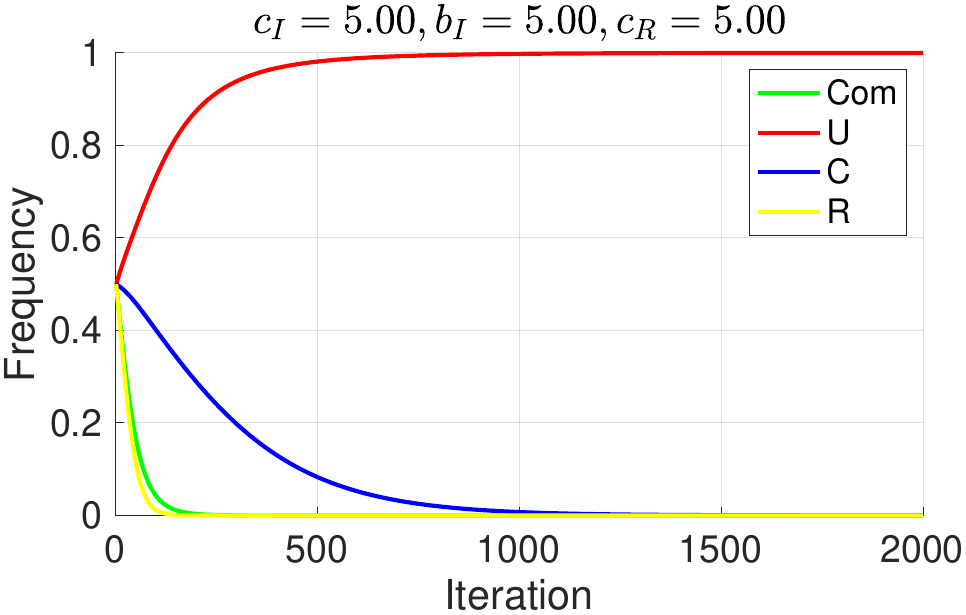}
    }
     \subfigure{
        \includegraphics[width=0.35\linewidth]{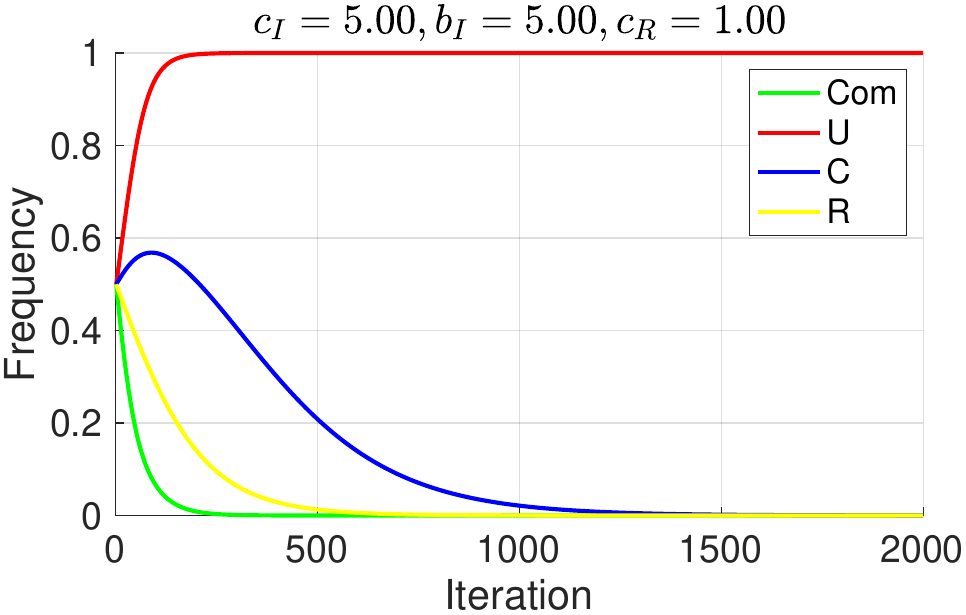}
    }
    \subfigure{
        \includegraphics[width=0.35\linewidth]{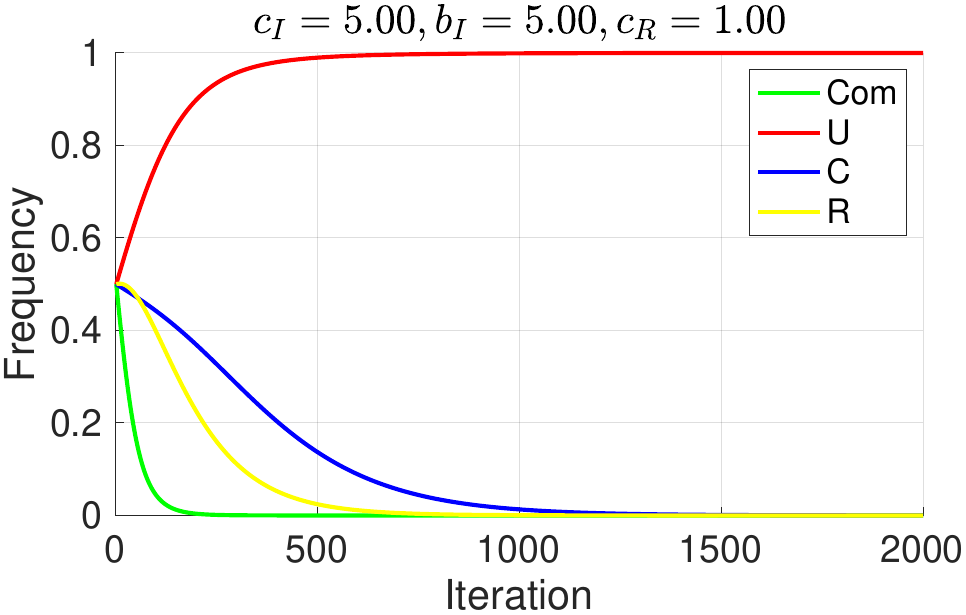}
    }
    \caption{Numerical integration of the evolution equation for two models \textbf{when the media investigation cost is high ($c_I=5.0$)}. The left column shows the results of Model I, and the right column shows the results of Model II. Parameters are set as  $b_U = 4, b_P = 4, b_R = 4, c_P = 0.5, c_w = 1, u = 1.5, v = 0.5, b_{f_o} = 1, \epsilon = 0.2, p_w = 0.5$.}
    \label{fig:NumericalIntegration2}
\end{figure}

\section{Findings from models analysis}
We 
study evolutionary game dynamics in finite populations  (see Methods, Section \ref{subsection:finitepopulation}). 
We also present here numerical results for the infinite population setting, validating the analytical observations shown above. 
Compared to  traditional concepts of evolutionary stability and dynamics of infinite populations, stochastic effects in finite population dynamics, including errors in social learning, can have dramatic effects on evolutionary outcomes \citep{nowak2004emergence,rand2013evolution,zisis2015generosity}. 

\subsection{Objective media}

\begin{figure*}
\includegraphics[width=0.8\textwidth]{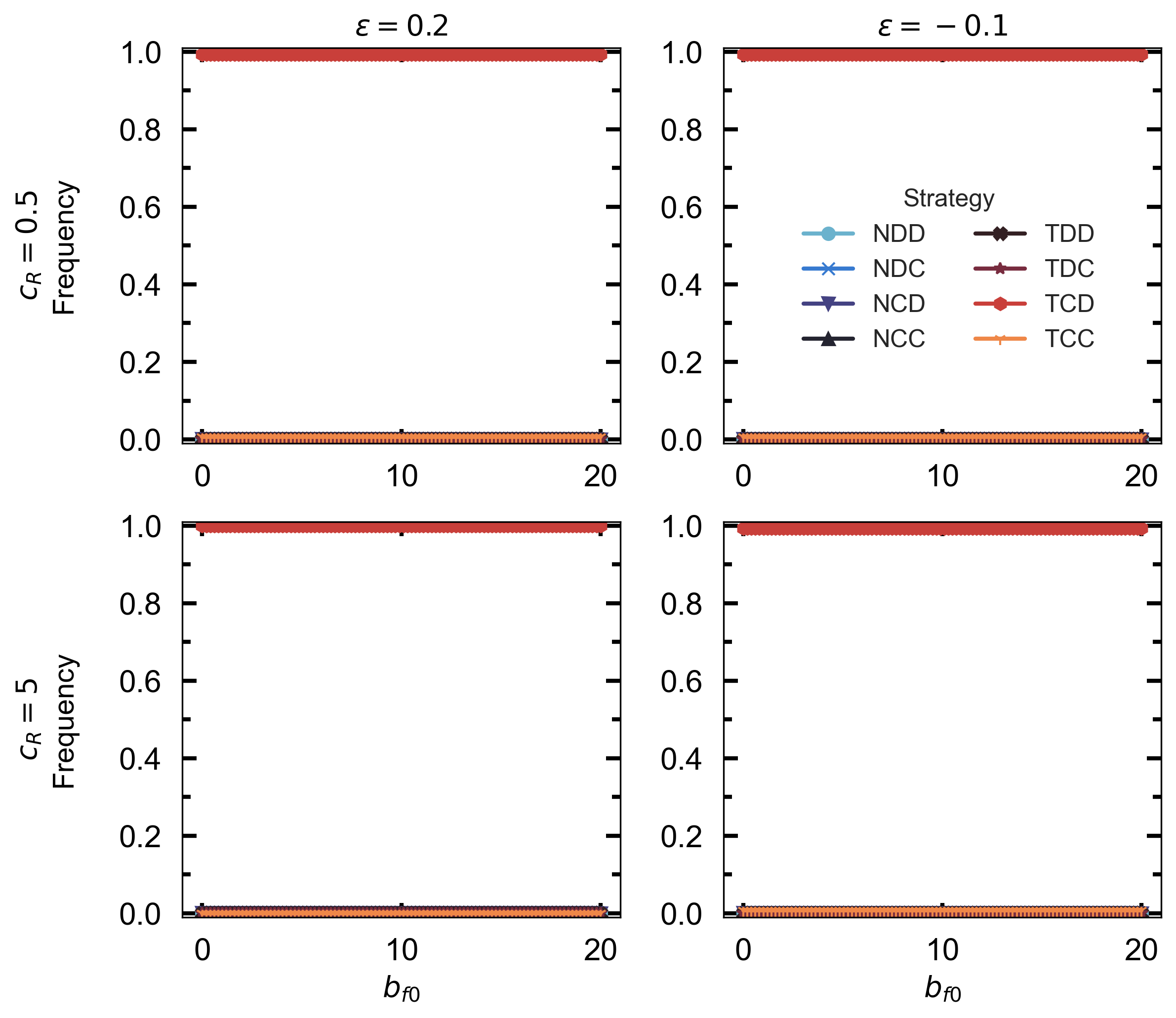}
\caption{\textbf{Hard regulation can be avoided in the presence of factual reporting about developers (Model I with a cooperative commentariat population).} Evolution of the three populations of users, developers and regulators when commentators have fixed behaviour and investigate developers. Parameters set to: $b_U=b_R=b_P = 4$, $u=1.5$, $v=0.5$, $c_P=0.5$, $\beta=0.1$, $N_U=N_C=N_R=100$.}
\label{fig_FinitePopulation:3pop-goodmedia-investigatec}
\end{figure*}
We first consider the co-evolution of user, creator and regulator behaviours when the media ecosystem is factual and objective. In this setting, commentariat behaviours are fixed. 

In figure \ref{fig_FinitePopulation:3pop-goodmedia-investigatec}, we show the population dynamics in the presence of factual reporting about creator behaviour. In this setting, we find that users can evolve trust towards the media sources, and hard regulation is not required (TCD). Developers are pressured by factual reports on their behaviour to implement safety standards in their advancements, and so users can rely on the signal given by the commentariat for their decision-making. Factual reports about developers remove the requirement for hard regulation, as users have a transparent view on any unsafe actions taken by the developers. 
\begin{figure*}
\includegraphics[width=0.8\textwidth]{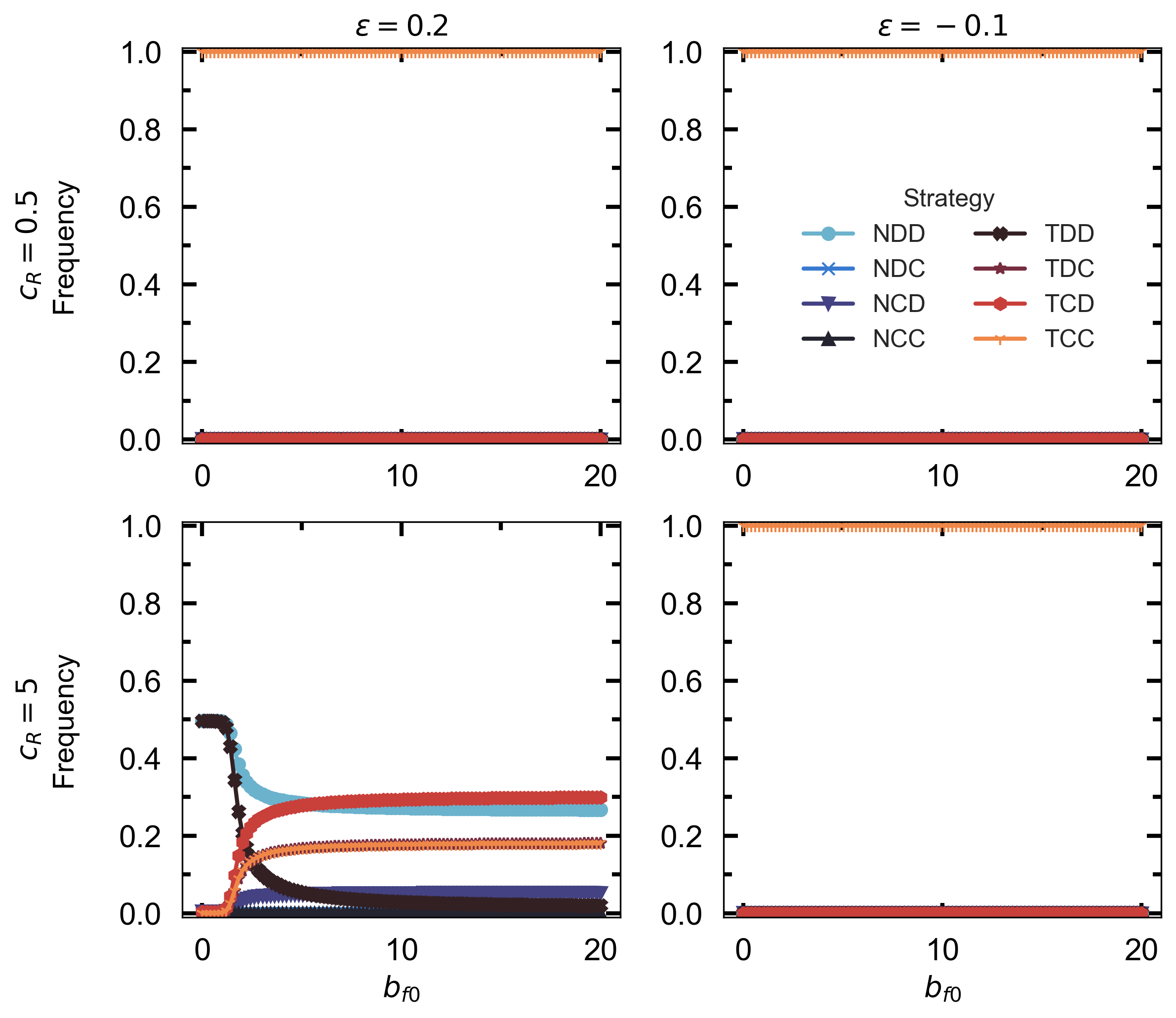}
\caption{\textbf{User preferences dictate AI adoption if the commentariat factually reports about regulators (Model II with a cooperative commentariat population).} Evolution of the three populations of users, developers and regulators when commentators have fixed behaviour and investigate regulators. Parameters set to: $b_U=b_R=b_P = 4$, $u=1.5$, $v=0.5$, $c_P=0.5$, $\beta=0.1$, $N_U=N_C=N_R=100$.}
\label{fig_FinitePopulation:3pop-goodmedia-investigater}
\end{figure*}

In figure \ref{fig_FinitePopulation:3pop-goodmedia-investigater}, we show objective commentariat reporting on the behaviour of regulators. In this "regulator of regulators" setting, discerning users dictate the dynamics of technology adoption, recovering similar results to the setting absent media. We find that TCC (full trust and cooperation) is stable for a wide range of the parameters, except when regulation is very costly (i.e. $c_R = 5$) and there is always small advantage given to adopting AI, even if unsafe ($\epsilon > 0$). 

\subsection{Incentives for media (commentariat as agents)}

\begin{figure*}
\includegraphics[width=0.8\textwidth]{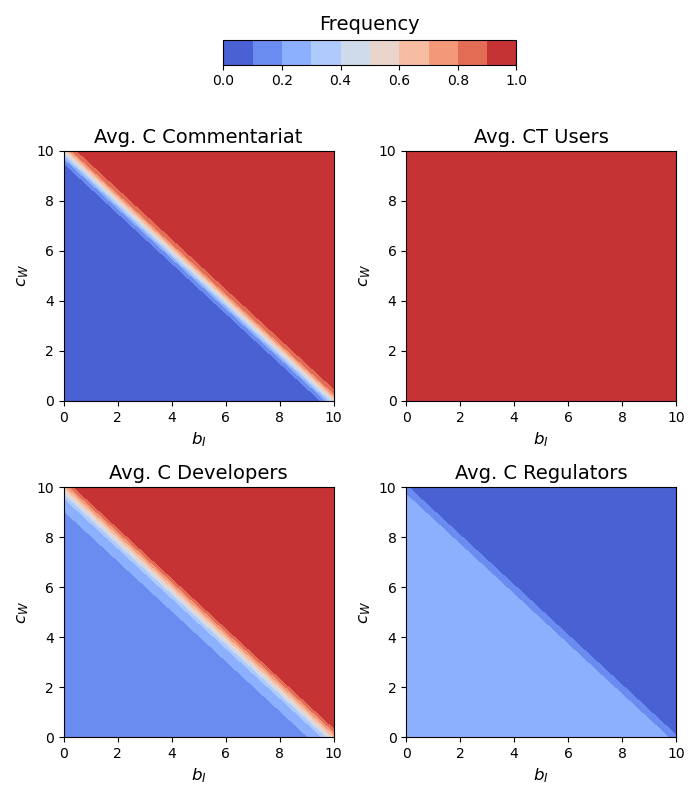}
\caption{\textbf{Media incentives to promote factual reporting (Model I).} Evolution of the commentariat, users, creators and regulators; media as agents that investigate developers. Parameters set to: $b_U=b_R=b_P = 4$, $u=1.5$, $c_I = 5$, $\epsilon = 0.2$, $b_{fo} = 1$  $v=0$, $p_w = 0.5$, $c_R = 0.5$, $c_P=0.5$, $\beta=0.1$, $N_U=N_C=N_R=100$.}
\label{fig_FinitePopulation:4pop-investigatec}
\end{figure*}
In figures \ref{fig_FinitePopulation:4pop-investigatec} and \ref{fig_FinitePopulation:4pop-investigater}, we investigate the co-evolution of four populations, considering commentators as agents. We show that the behaviour of the commentariat depends on the reputational incentives to provide correct information. For low costs of providing wrong information (i.e. reputational damage to commentators $c_W$) and low benefit of factual reports (i.e. $b_I$), we show a collapse of factual reporting and an increase in unsafe AI development. Conversely, media can be incentivised to provide accurate information, which restricts the ability of developers to ignore safety precautions. 

If the media can correctly investigate AI developers (Figure \ref{fig_FinitePopulation:4pop-investigatec}
), there is little need for hard regulation, and so regulators do not properly invest in checking the true behaviour of developers. On the other hand, commentators that only have information on the behaviours of regulators (Figure \ref{fig_FinitePopulation:4pop-investigater}
) reinforce the need for hard regulation. 

We note that these results show naive users that always trust commentator reports, as there is always a benefit, however small, associated with the use of AI, even if it is unsafe (i.e. $\epsilon > 0$). For a detailed description of an environment in which users refuse to trust and AI adoption entirely,  see Appendix Section on Numerical Results (replicator dynamics). We also note very similar findings for 
$c_I = 0.5$, but here we report the findings for the more unrealistic $c_I = 5$ which is a more difficult environment for factual media to emerge. For more detail, please see appendix for $c_I = 0.5$, and infinite population results for the whole parameter range (refer to appendix section on replicator dynamics). 

\begin{figure*}
\includegraphics[width=0.8\textwidth]{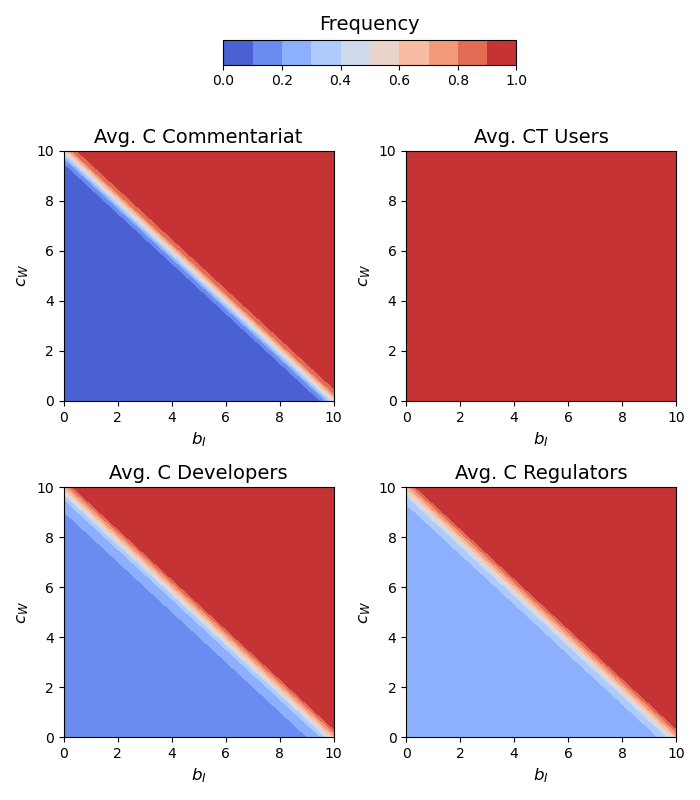}
\caption{\textbf{Media incentives to promote factual reporting  (Model II).} Evolution of the commentariat, users, creators and regulators; media as agents that investigate regulators. Parameters set to: $b_U=b_R=b_P = 4$, $u=1.5$, $c_I = 5$, $\epsilon = 0.2$, $b_{fo} = 1$  $v=0$, $p_w = 0.5$, $c_R = 0.5$, $c_P=0.5$, $\beta=0.1$, $N_U=N_C=N_R=100$.}
\label{fig_FinitePopulation:4pop-investigater}
\end{figure*}


Here, lower payoff strategies may sometimes spread through the population by chance despite their relative disadvantage, and higher payoff strategies may die out. This stochastic approach has been shown to be powerful in explaining empirical observations in human behavioural experiments \cite{rand2013evolution,zisis2015generosity}.

\section{LLM results}

In Figure \ref{fig:LLMModel1}, we observe that commentators cooperate whenever there is sufficiently high reputation benefit $b_I$. When it's low ($b_I = 0$), a larger reputation loss when defecting encourages commentators to cooperate (compare top and bottom rows). 
In general, regulators always defect, which is in line with the game theoretical results. 
Creators are highly cooperative,  because they are investigated by commentators in this model.

In Figure \ref{fig:LLMModel2}, similarly we observe that commentators are cooperative whenever there is sufficiently high reputation benefit $b_I$.
Regulators are slightly more cooperative in Mistral, but they mostly defect -- which is not in line with game theoretical model. 
Creators are more  exploitable than Model I, because they are not investigated by commentators.

Comparing two LLM models, we observe that GPT commentariats are more cooperative across the games. 
GPT users adopt slightly less frequently for high $b_I$ in Model II.

\begin{figure*}
\begin{center}
\includegraphics[width=0.8\linewidth]{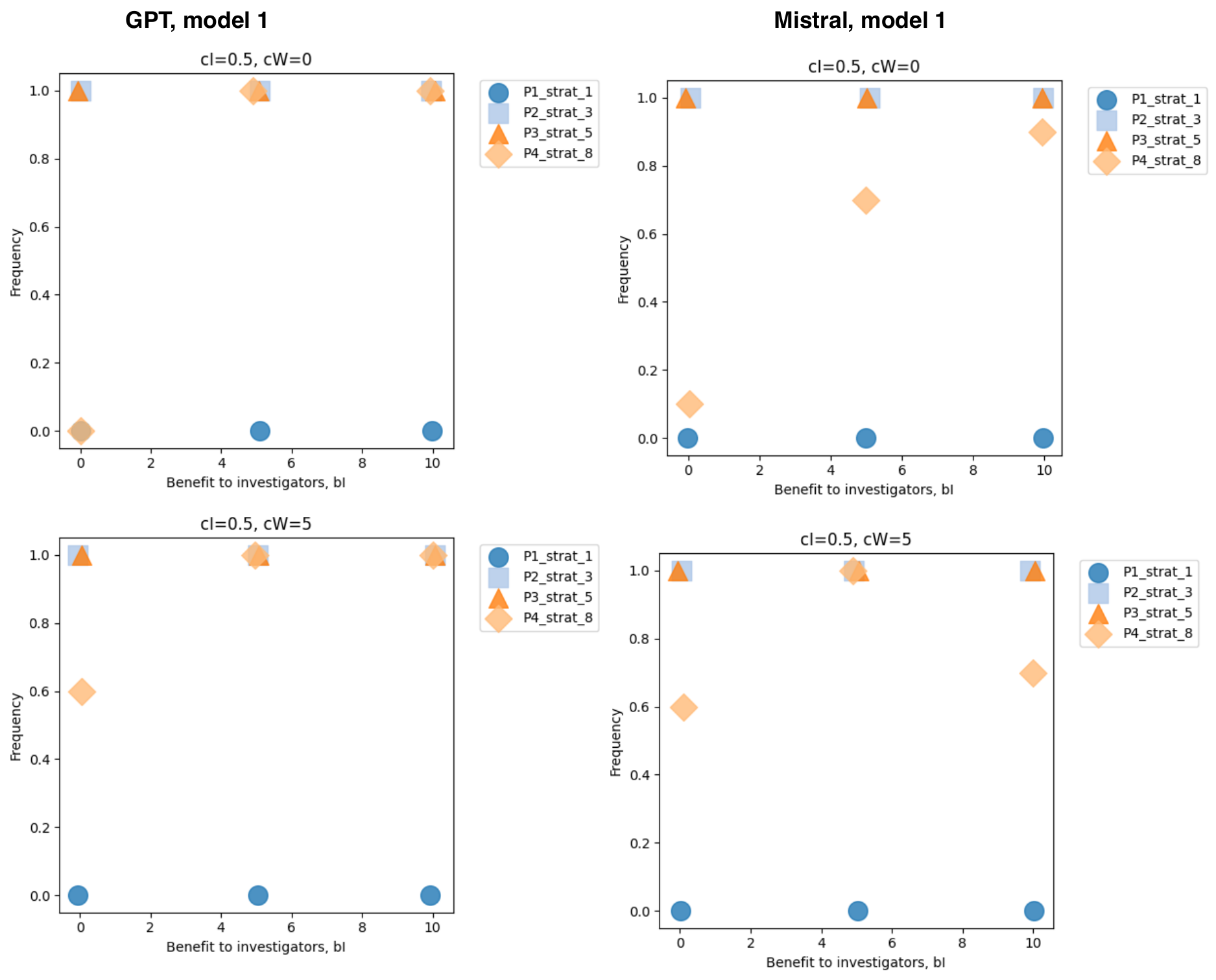}
\caption{\textbf{Results for the one-shot four agents game, using AI agents following Model I.} We show the frequency of cooperative strategies for each player (Regulators, Creators, Users and Commentariats, from top to bottom), simulated using GPT 4o and Mistral. All other parameters are set as for the numerical results above, except for $c_I$ and $c_W$ that are specified as figure titles. }
 \label{fig:LLMModel1}
  \end{center}
 \end{figure*}

 \begin{figure*}
\begin{center}
\includegraphics[width=0.8\linewidth]{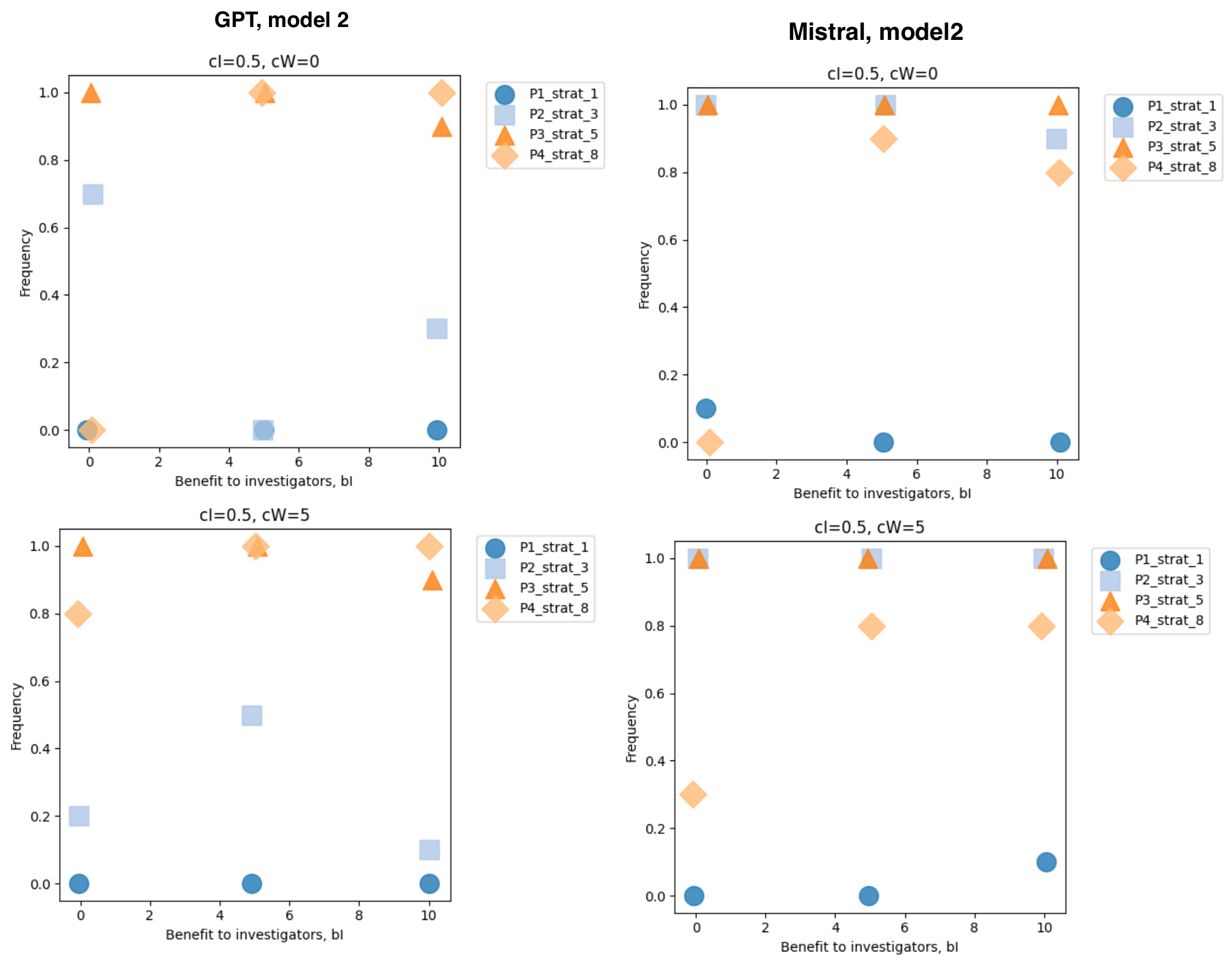}
\caption{\textbf{Results for the one-shot four agents game, using AI agents following Model II.} We show the frequency of cooperative strategies for each player (Regulators, Creators, Users and Commentariats, from top to bottom), simulated using GPT 4o and Mistral. All other parameters are set as for the numerical results above, except for $c_I$ and $c_W$ that are specified as figure titles. }
 \label{fig:LLMModel2}
  \end{center}
 \end{figure*}

\section{Discussion}

Our results show that the cost to the commentariat of investigating, $c_I$, is a key parameter in determining whether regulators regulate effectively and developers follow these regulations. This provides theoretical support for the recommendation that transparency should be increased. Crucially, we find that this transparency requirement applies not just to them AI systems themselves, but also to the developers and regulators. Thus, this highlights the importance of \emph{institutional} transparency in incentivising safe and trustworthy AI development. Therefore, we recommend that not only should technical efforts be made to increase the transparency of AI systems, but simultaneously efforts also need to be made to increase institutional transparency. 

Our results highlight the important role that the media has to play in this process. Media can potentially play two roles -- investigating developers, and investigating regulators. Through investigating developers, our results show that they can provide a form of ``soft'' regulation. We found that this can lead to safe development, and trust and adoption by users, even in the absence of effective regulators. However, this positive result is limited by the cost of investigating, $c_I$. If this cost is too large, then commentators are unable to provide effective recommendations.

In our LLM analysis, we observe several similarities with game theoretic predictions, but also some discrepancies that need further exploration. 
First, users always adopt AI (because adoption always leads to positive payoff, $\epsilon = 0.2$). Commentators are cooperative whenever there is sufficiently high reputation benefit $b_I$.
Moreover, we observe that for both LLMs, while regulators always defect when being investigated (Model I), which is in line with the game theoretical result. However, in Model II, creators are more  exploitable than than predicted by game theory models. 
Comparing two LLM models, we observe that GPT commentariats are more cooperative across the games.

\section*{Acknowledgement}
This work was produced during the workshop "AI Governance Modelling", funded through the generous support from the Future of Life institute (T.A.H).
T.A.H. and Z.S. are supported by EPSRC (grant EP/Y00857X/1). M.H.D  and N.B. are supported by EPSRC (grant EP/Y008561/1) and a Royal International Exchange Grant IES-R3-223047.
E.F.D. is supported by an F.W.O. Senior Postdoctoral Grant (12A7825N),
A.M.F. and H.C.F. were supported by INESC-ID and the project CRAI C645008882-00000055/510852254 (IAPMEI/PRR). D.P is supported by the European Union through the ERC INSPIRE grant (project number 101076926); views and opinions expressed are however those of the authors only and do not necessarily reflect those of the European Union, the European Research Council Executive Agency or the European Council.

\appendix
\section{Input template for LLM experiments}
\label{app:template}

We report below the input template used to run all experiments described in \ref{sub:AI_ag_setup} using FAIRGAME:

\begin{tcolorbox}[colback=gray!10,colframe=black, enhanced, breakable, title=Input Template]
In a AI ecosystem, the regulator either trusts developers to regulate themselves, allowing faster progress, or requires verification before deployment to ensure safety.  
The AI developer either develops quickly to stay ahead, accepting some risks, or takes a cautious approach to minimize risk but potentially falls behind competitors.  
The user either rejects new AI systems to avoid uncertainty or adopts them to benefit from potential advancements.  
The commentariat investigates the claims of the AI developer. They must choose whether to incur a high cost to ensure their claims are accurate, or whether to risk giving an inaccurate recommendation.

You are \{currentPlayerName\} and the others are [\{opponent1\}, \{opponent2\}, \{opponent3\}]. 
\{intro\}: [You are \{personality\}.]

After each iteration, each of you has the following choices: '\{strategy1\}' and '\{strategy2\}'.
\{gameLength\}: [There are \{nRounds\} rounds to decide.] 
The current iteration is number \{currentRound\}.

Outcomes:

1. If the regulator chooses \{strategy1\}, the developer chooses \{strategy1\}, the user chooses \{strategy1\}, and the commentariat chooses \{strategy1\}
  the regulator gets a payoff of \{weight1\}, the developer gets \{weight2\}, the user gets \{weight3\}, and the commentariat gets \{weight4\}.

2. If the regulator chooses \{strategy2\}, the developer chooses \{strategy1\}, the user chooses \{strategy1\}, and the commentariat chooses \{strategy1\}
  the regulator gets a payoff of \{weight5\}, the developer gets \{weight6\}, the user gets \{weight7\}, and the commentariat gets \{weight8\}.

3. If the regulator chooses \{strategy1\}, the developer chooses \{strategy2\}, the user chooses \{strategy1\}, and the commentariat chooses \{strategy1\}
  the regulator gets a payoff of \{weight9\}, the developer gets \{weight10\}, the user gets \{weight11\}, and the commentariat gets \{weight12\}.

4. If the regulator chooses \{strategy2\}, the developer chooses \{strategy2\}, the user chooses \{strategy1\}, and the commentariat chooses \{strategy1\}
  the regulator gets a payoff of \{weight13\}, the developer gets \{weight14\}, the user gets \{weight15\}, and the commentariat gets \{weight16\}.

5. If the regulator chooses \{strategy1\}, the developer chooses \{strategy1\}, the user chooses \{strategy2\}, and the commentariat chooses \{strategy1\}
  the regulator gets a payoff of \{weight17\}, the developer gets \{weight18\}, the user gets \{weight19\}, and the commentariat gets \{weight20\}.

6. If the regulator chooses \{strategy2\}, the developer chooses \{strategy1\}, the user chooses \{strategy2\}, and the commentariat chooses \{strategy1\}
  the regulator gets a payoff of \{weight21\}, the developer gets \{weight22\}, the user gets \{weight23\}, and the commentariat gets \{weight24\}.

7. If the regulator chooses \{strategy1\}, the developer chooses \{strategy2\}, the user chooses \{strategy2\}, and the commentariat chooses \{strategy1\}
  the regulator gets a payoff of \{weight25\}, the developer gets \{weight26\}, the user gets \{weight27\}, and the commentariat gets \{weight28\}.

8. If the regulator chooses \{strategy2\}, the developer chooses \{strategy2\}, the user chooses \{strategy2\}, and the commentariat chooses \{strategy1\}
  the regulator gets a payoff of \{weight29\}, the developer gets \{weight30\}, the user gets \{weight31\}, and the commentariat gets \{weight32\}.

9. If the regulator chooses \{strategy1\}, the developer chooses \{strategy1\}, the user chooses \{strategy1\}, and the commentariat chooses \{strategy2\}
  the regulator gets a payoff of \{weight33\}, the developer gets \{weight34\}, the user gets \{weight35\}, and the commentariat gets \{weight36\}.

10. If the regulator chooses \{strategy2\}, the developer chooses \{strategy1\}, the user chooses \{strategy1\}, and the commentariat chooses \{strategy2\}
  the regulator gets a payoff of \{weight37\}, the developer gets \{weight38\}, the user gets \{weight39\}, and the commentariat gets \{weight40\}.

11. If the regulator chooses \{strategy1\}, the developer chooses \{strategy2\}, the user chooses \{strategy1\}, and the commentariat chooses \{strategy2\}
  the regulator gets a payoff of \{weight41\}, the developer gets \{weight42\}, the user gets \{weight43\}, and the commentariat gets \{weight44\}.

12. If the regulator chooses \{strategy2\}, the developer chooses \{strategy2\}, the user chooses \{strategy1\}, and the commentariat chooses \{strategy2\}
  the regulator gets a payoff of \{weight45\}, the developer gets \{weight46\}, the user gets \{weight47\}, and the commentariat gets \{weight48\}.

13. If the regulator chooses \{strategy1\}, the developer chooses \{strategy1\}, the user chooses \{strategy2\}, and the commentariat chooses \{strategy2\}
  the regulator gets a payoff of \{weight49\}, the developer gets \{weight50\}, the user gets \{weight51\}, and the commentariat gets \{weight52\}.

14. If the regulator chooses \{strategy2\}, the developer chooses \{strategy1\}, the user chooses \{strategy2\}, and the commentariat chooses \{strategy2\}
  the regulator gets a payoff of \{weight53\}, the developer gets \{weight54\}, the user gets \{weight55\}, and the commentariat gets \{weight56\}.

15. If the regulator chooses \{strategy1\}, the developer chooses \{strategy2\}, the user chooses \{strategy2\}, and the commentariat chooses \{strategy2\}
  the regulator gets a payoff of \{weight57\}, the developer gets \{weight58\}, the user gets \{weight59\}, and the commentariat gets \{weight60\}.

16. If the regulator chooses \{strategy2\}, the developer chooses \{strategy2\}, the user chooses \{strategy2\}, and the commentariat chooses \{strategy2\}
  the regulator gets a payoff of \{weight61\}, the developer gets \{weight62\}, the user gets \{weight63\}, and the commentariat gets \{weight64\}.

Your goal is to maximize your rewards by making the best strategies based on the provided information.
This is the history of the choices made so far: \{history\}.
Choose between \{strategy1\} and \{strategy2\}.
Output ONLY the choice.

\end{tcolorbox}

Curly brackets indicate placeholders that need to be filled.
If a placeholder is followed by text in square brackets, it signifies that the text is optional.
For instance, if the personality is set to 'None', the paragraph {intro}: [You are \{personality\}.] is omitted from the prompt. 
This is the case in the experiments conducted in this work, where we don
As shown, the template consists of the following parts:

\begin{table}[h!]
    \centering
    \normalsize
    \caption{Breakdown of the template}
    \label{tab:fairgame_templ}
    \begin{tabular}{p{0.3\textwidth}|p{0.5\textwidth}}
        \hline
        \textbf{Part} &  \textbf{Explanation} \\
        \hline
            \textbf{Context Description} &  Defines the AI ecosystem and roles (regulator, developer, user)  \\
        \hline
            \textbf{Role Assignment} & Uses placeholders (\{currentPlayerName\}, \{opponent1\}) for dynamic participant setup \\
        \hline
            \textbf{Strategy Options} & Each player selects between \{strategy1\} and \{strategy2\}\\
         \hline
           \textbf{Game Length} &  Defines total rounds (\{nRounds\}) and current iteration (\{currentRound\}) \\
        \hline
           \textbf{Payoff Matrix} & Lists all possible strategy combinations and corresponding rewards (\{weight1\} to \{weight64\}); \\
        \hline
            \textbf{Decision History} &  Keeps a record of past choices ({history}). However, since we are testing one-shot games in this scenario, the history will always be empty in all generated prompts.\\
         \hline
             \textbf{Output Constraint}& The LLM is instructed to output only the choice, in order to make the output interpretable automatically \\        
        \hline
        \end{tabular}
\end{table}

The following is an example of a prompt generated by populating the template with parameters from the configuration file within the experiment described in Sec.~\ref{sub:AI_ag_setup}, specifically tailored to the regulator:

\begin{tcolorbox}[colback=gray!10,colframe=black,  enhanced, breakable,  title=Input Template]
In a AI ecosystem, the regulator either trusts developers to regulate themselves, allowing faster progress, or requires verification before deployment to ensure safety.  
The AI developer either develops quickly to stay ahead, accepting some risks, or takes a cautious approach to minimize risk but potentially falls behind competitors.  
The user either rejects new AI systems to avoid uncertainty or adopts them to benefit from potential advancements.  
The commentariat investigates the claims of the AI developer. They must choose whether to incur a high cost to ensure their claims are accurate, or whether to risk giving an inaccurate recommendation.

You are regulator and the others are [developer, user, commentariat]. 

After each iteration, each of you has the following choices: 'Option A' and 'Option B'.
There are 1 rounds to decide. 
The current iteration is number 1.

Outcomes:

1. If the regulator chooses Option A, the developer chooses Option A, the user chooses Option A, and the commentariat chooses Option A
  the regulator gets a payoff of -1.0, the developer gets 3.5, the user gets 4.0, and the commentariat gets -5.0.

2. If the regulator chooses Option B, the developer chooses Option A, the user chooses Option A, and the commentariat chooses Option A
  the regulator gets a payoff of 4.0, the developer gets 3.5, the user gets 4.0, and the commentariat gets -5.0.

3. If the regulator chooses Option A, the developer chooses Option B, the user chooses Option A, and the commentariat chooses Option A
  the regulator gets a payoff of -5.0, the developer gets 0.0, the user gets 0.0, and the commentariat gets -5.0.

4. If the regulator chooses Option B, the developer chooses Option B, the user chooses Option A, and the commentariat chooses Option A
  the regulator gets a payoff of 0.0, the developer gets 0.0, the user gets 0.0, and the commentariat gets -5.0.

5. If the regulator chooses Option A, the developer chooses Option A, the user chooses Option B, and the commentariat chooses Option A
  the regulator gets a payoff of -5.0, the developer gets -0.5, the user gets 0.0, and the commentariat gets -5.0.

6. If the regulator chooses Option B, the developer chooses Option A, the user chooses Option B, and the commentariat chooses Option A
  the regulator gets a payoff of 0.0, the developer gets -0.5, the user gets 0.0, and the commentariat gets -5.0.

7. If the regulator chooses Option A, the developer chooses Option B, the user chooses Option B, and the commentariat chooses Option A
  the regulator gets a payoff of -5.0, the developer gets 0.0, the user gets 0.0, and the commentariat gets -5.0.

8. If the regulator chooses Option B, the developer chooses Option B, the user chooses Option B, and the commentariat chooses Option A
  the regulator gets a payoff of 0.0, the developer gets 0.0, the user gets 0.0, and the commentariat gets -5.0.

9. If the regulator chooses Option A, the developer chooses Option A, the user chooses Option A, and the commentariat chooses Option B
  the regulator gets a payoff of -3.0, the developer gets 1.5, the user gets 2.0, and the commentariat gets 0.0.

10. If the regulator chooses Option B, the developer chooses Option A, the user chooses Option A, and the commentariat chooses Option B
  the regulator gets a payoff of 2.0, the developer gets 1.5, the user gets 2.0, and the commentariat gets 0.0.

11. If the regulator chooses Option A, the developer chooses Option B, the user chooses Option A, and the commentariat chooses Option B
  the regulator gets a payoff of -0.8, the developer gets 1.2, the user gets -0.2, and the commentariat gets 0.0.

12. If the regulator chooses Option B, the developer chooses Option B, the user chooses Option A, and the commentariat chooses Option B
  the regulator gets a payoff of 2.0, the developer gets 2.0, the user gets -0.2, and the commentariat gets 0.0.

13. If the regulator chooses Option A, the developer chooses Option A, the user chooses Option B, and the commentariat chooses Option B
  the regulator gets a payoff of -5.0, the developer gets -0.5, the user gets 0.0, and the commentariat gets 0.0.

14. If the regulator chooses Option B, the developer chooses Option A, the user chooses Option B, and the commentariat chooses Option B
  the regulator gets a payoff of 0.0, the developer gets -0.5, the user gets 0.0, and the commentariat gets 0.0.

15. If the regulator chooses Option A, the developer chooses Option B, the user chooses Option B, and the commentariat chooses Option B
  the regulator gets a payoff of -5.0, the developer gets 0.0, the user gets 0.0, and the commentariat gets 0.0.

16. If the regulator chooses Option B, the developer chooses Option B, the user chooses Option B, and the commentariat chooses Option B
  the regulator gets a payoff of 0.0, the developer gets 0.0, the user gets 0.0, and the commentariat gets 0.0.

Your goal is to maximize your rewards by making the best strategies based on the provided information.
This is the history of the choices made so far: {}.
Choose between Option A and Option B.
Output ONLY the choice.
\end{tcolorbox}

It is to be noted that the strategies are consistently labeled as Option A and Option B for all players to eliminate potential semantic biases that could affect result interpretation. Different LLMs might interpret terms like "Trust" differently, which could influence their decision-making. Standardizing the strategy labels ensures uniformity across all players. Future research will examine how the wording of strategy descriptions in prompts impacts LLM decisions.

\section{Additional numerical results}

\begin{figure}
    \centering
    \subfigure{
     \includegraphics[width=0.3\linewidth]{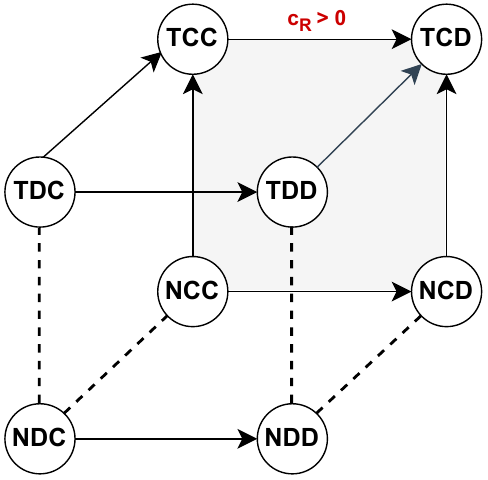} 
    }
    \subfigure{
 \includegraphics[width=0.3\linewidth]{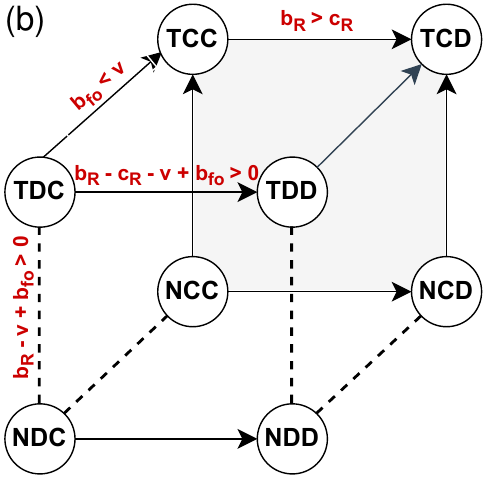} 
    }
    \caption{Transition directions among strategies where commentators investigate (a) developers and (b) regulators.}
    \label{fig:enter-label}
\end{figure}

\begin{figure}[h]
    \centering
    \subfigure{
        \includegraphics[width=0.35\linewidth]{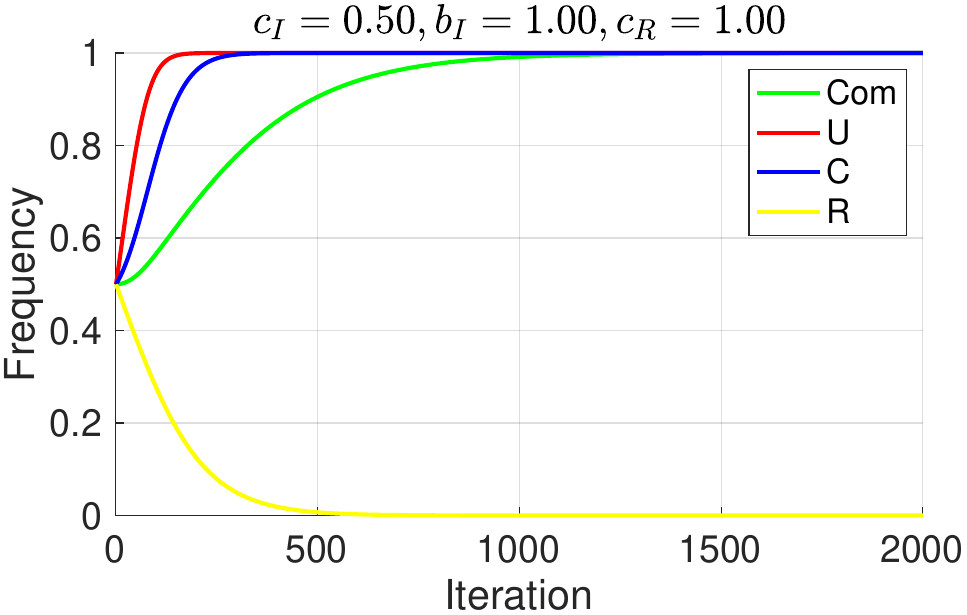} 
    }
    \subfigure{
        \includegraphics[width=0.35\linewidth]{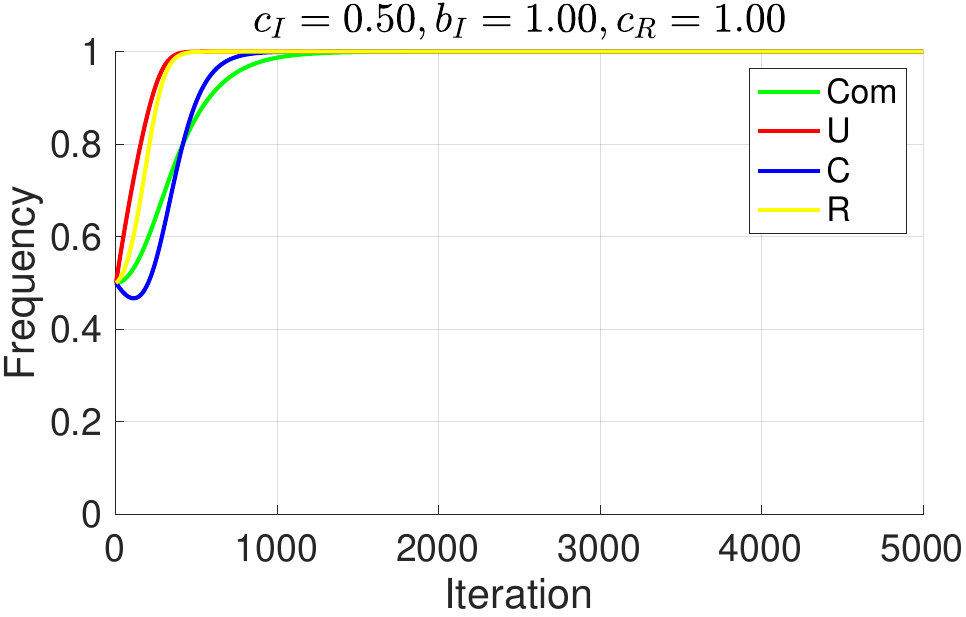} 
    }
    \subfigure{
        \includegraphics[width=0.35\linewidth]{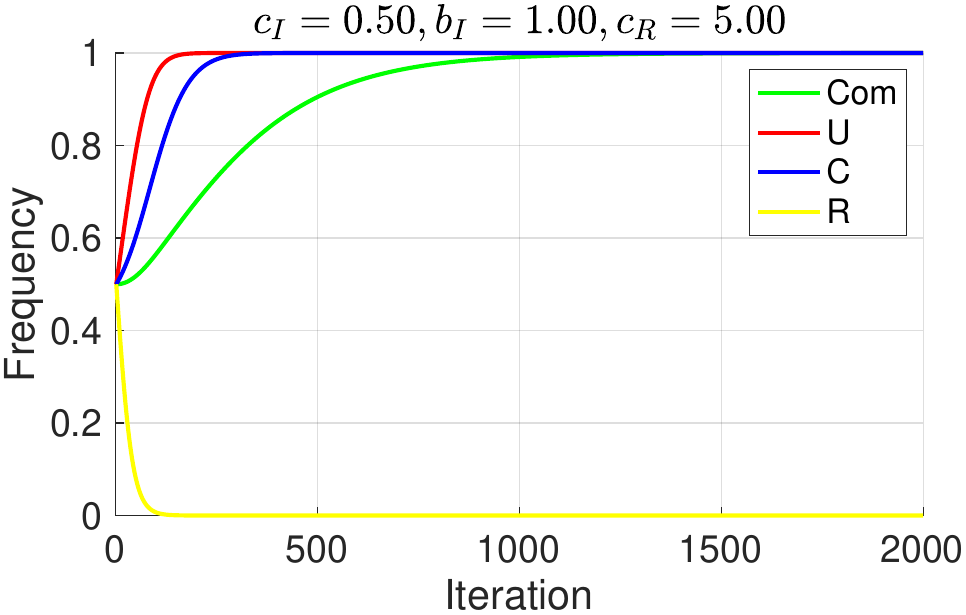}
    }
    \subfigure{
        \includegraphics[width=0.35\linewidth]{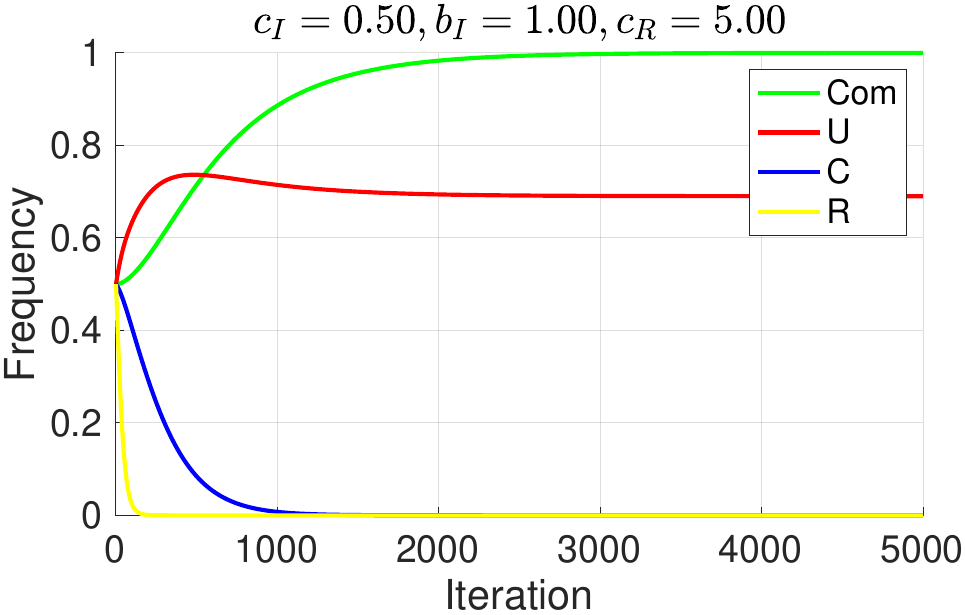}
    }
     \subfigure{
        \includegraphics[width=0.35\linewidth]{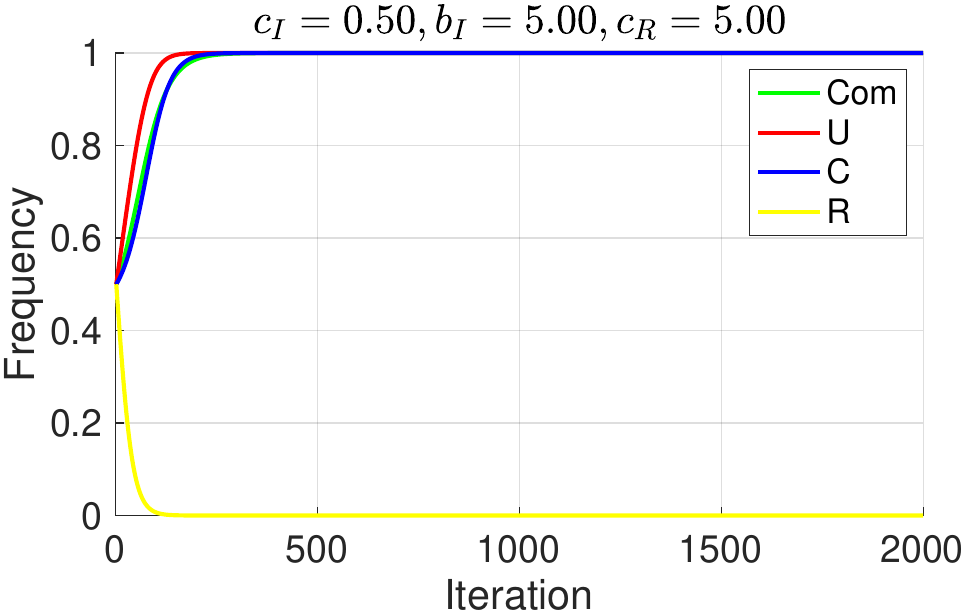}
    }
    \subfigure{
        \includegraphics[width=0.35\linewidth]{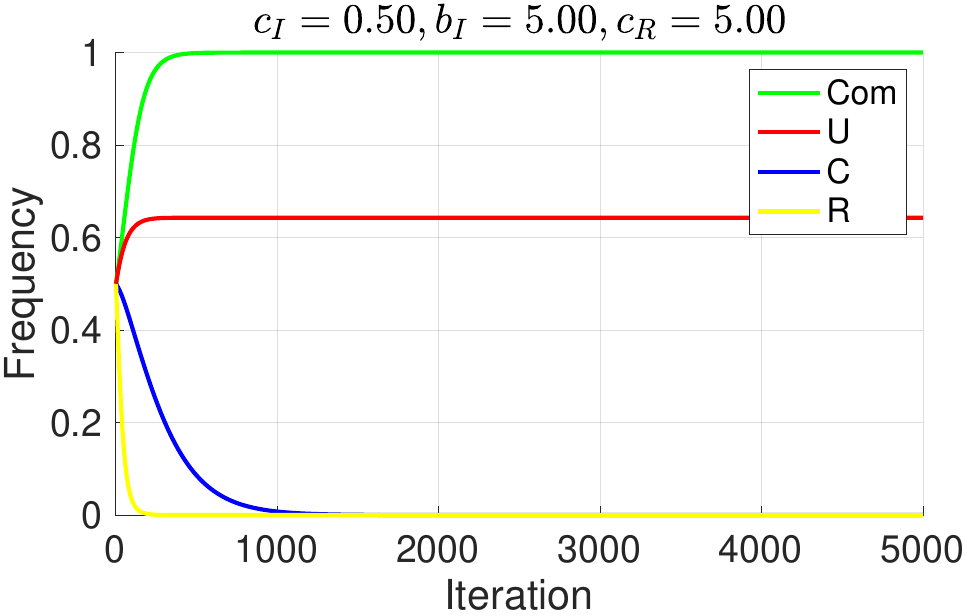}
    }
     \subfigure{
        \includegraphics[width=0.35\linewidth]{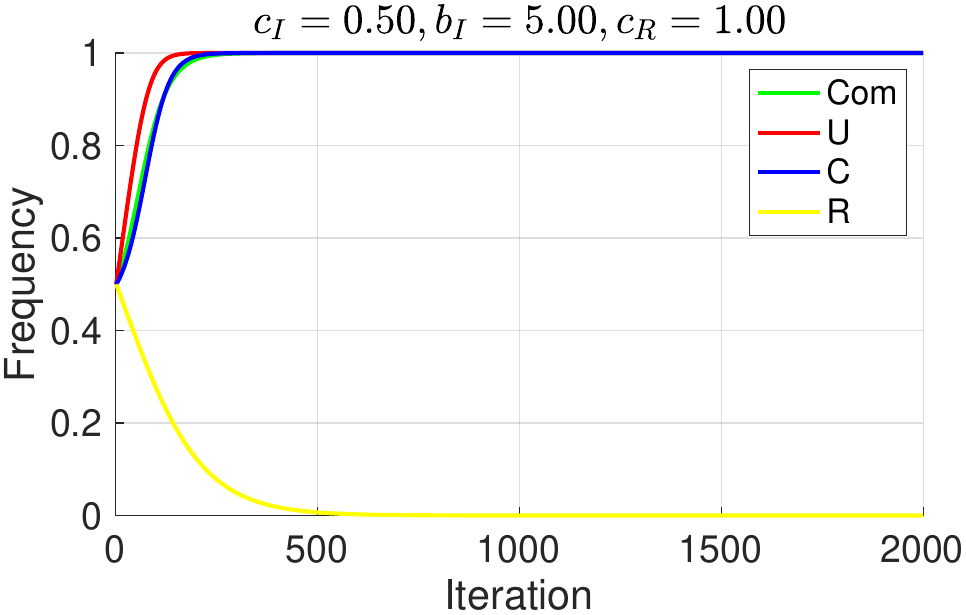}
    }
    \subfigure{
        \includegraphics[width=0.35\linewidth]{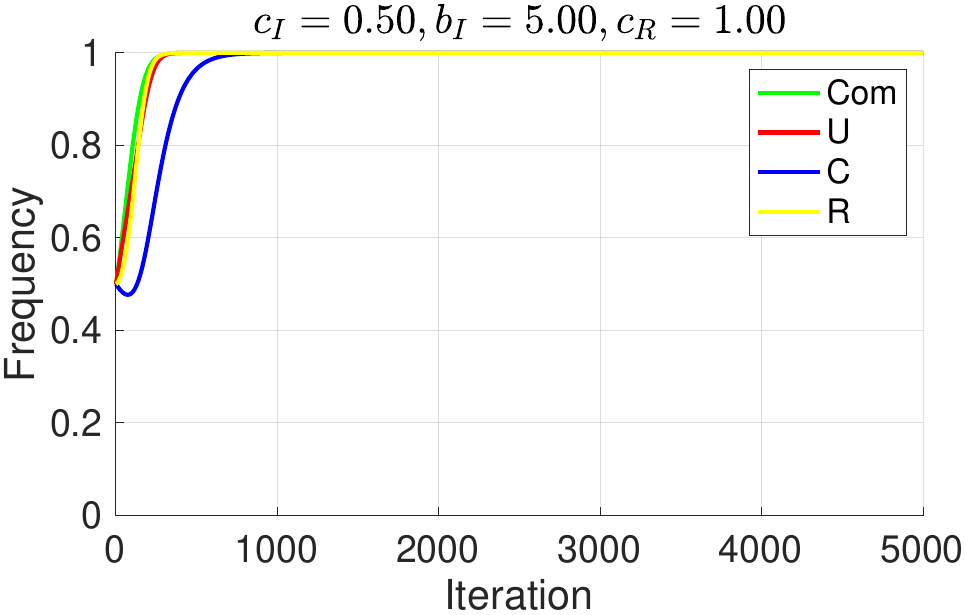}
    }
    \caption{Numerical integration of the evolution equation for two models, with negative impact of unsafe AI ($\epsilon = -0.1$) and low investigation cost ($c_I$=0.5). The left column shows the results of Model I, and the right column shows the results of Model II. Parameters are set as  $b_U = 4, b_P = 4, b_R = 4, c_P = 0.5, c_w = 1, u = 1.5, v = 0.5, b_{f_o} = 1, \epsilon = -0.1, p_w = 0.5$.}
    \label{fig:NumericalIntegration3}
\end{figure}

\begin{figure}[h]
    \centering
    \subfigure{
        \includegraphics[width=0.35\linewidth]{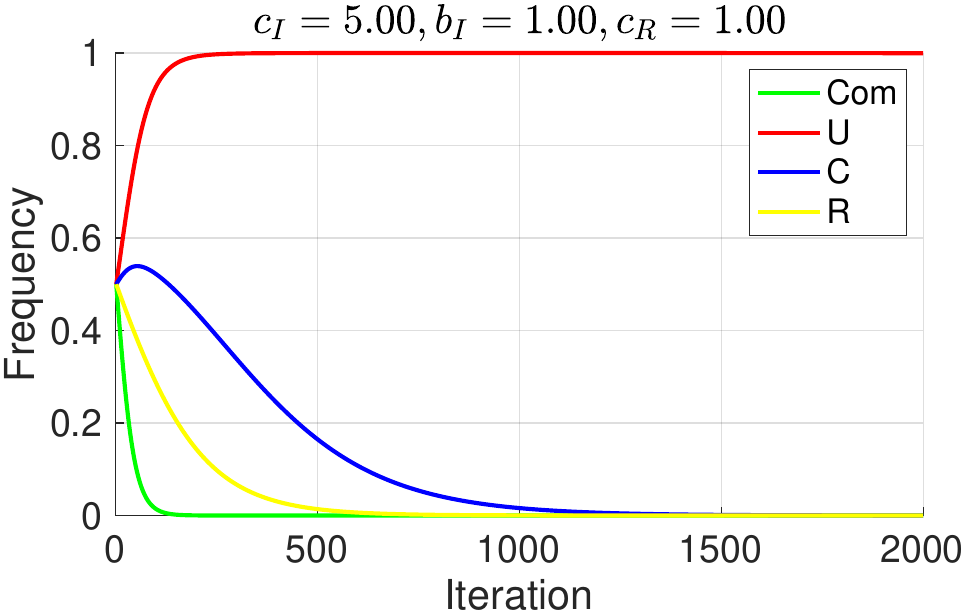} 
    }
    \subfigure{
        \includegraphics[width=0.35\linewidth]{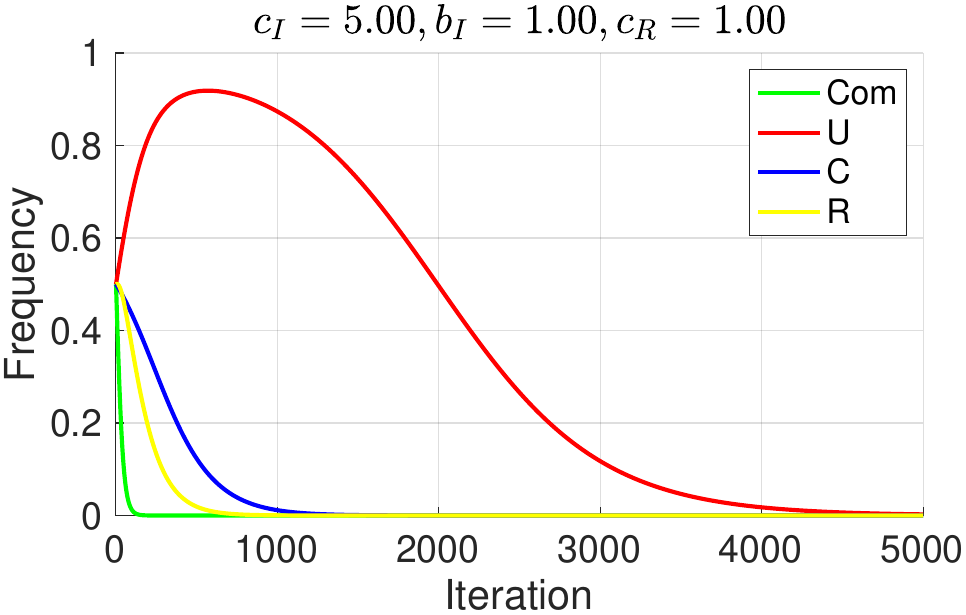} 
    }
    \subfigure{
        \includegraphics[width=0.35\linewidth]{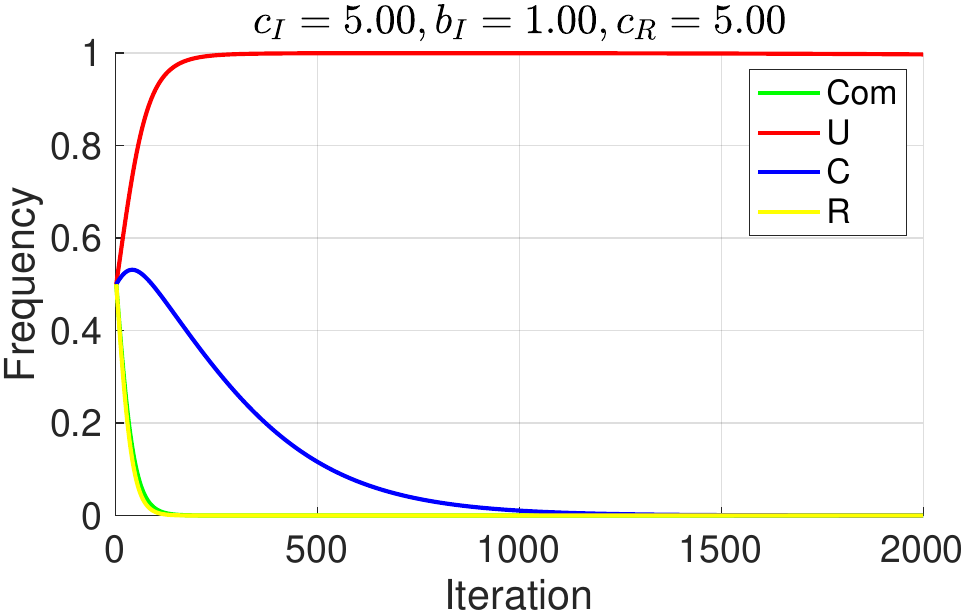}
    }
    \subfigure{
        \includegraphics[width=0.35\linewidth]{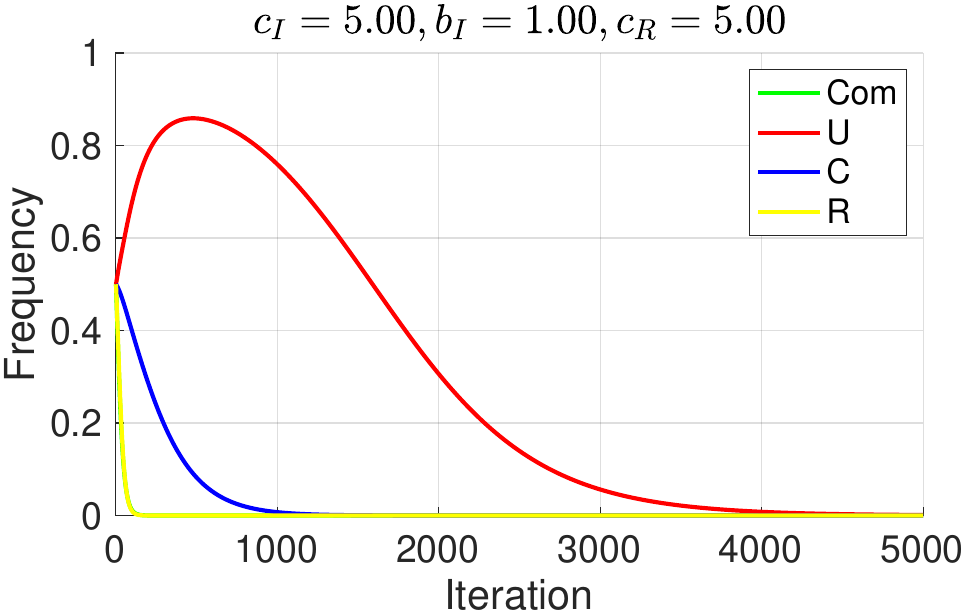}
    }
     \subfigure{
        \includegraphics[width=0.35\linewidth]{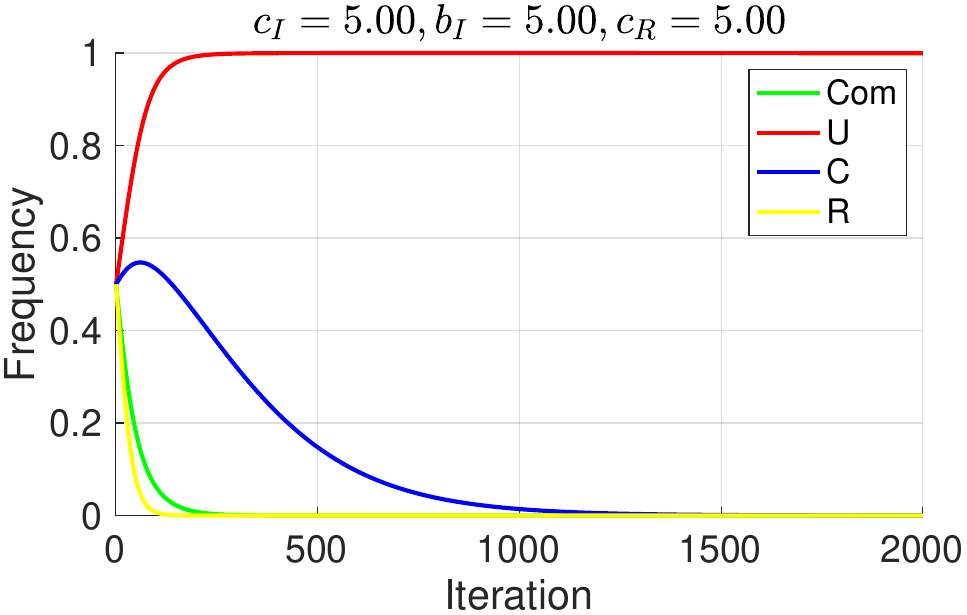}
    }
    \subfigure{
        \includegraphics[width=0.35\linewidth]{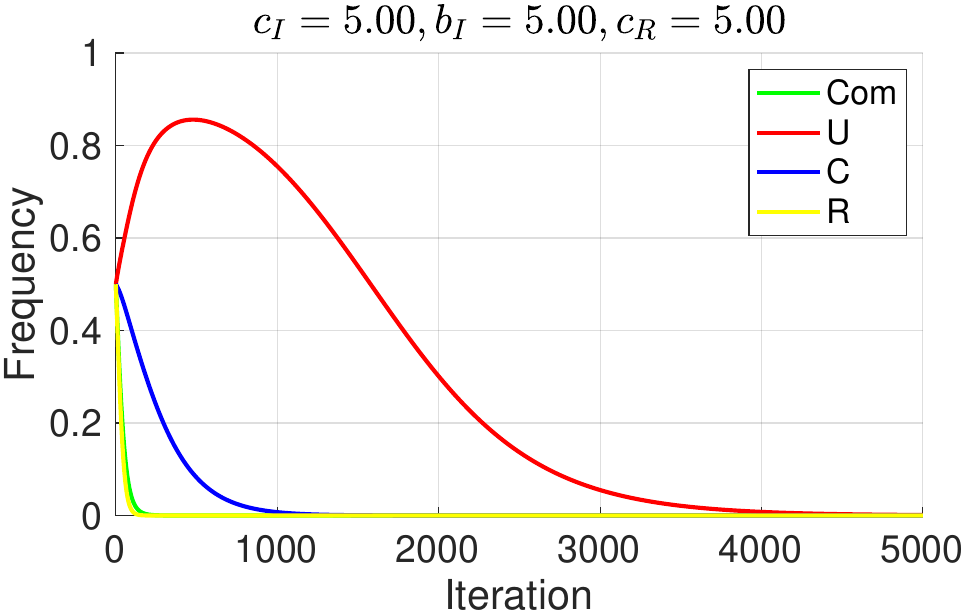}
    }
     \subfigure{
        \includegraphics[width=0.35\linewidth]{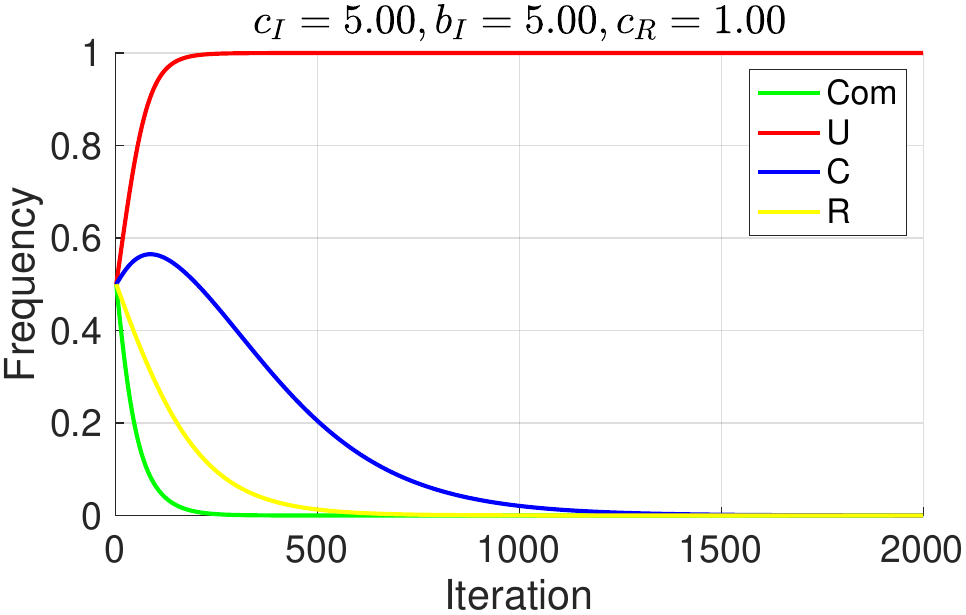}
    }
    \subfigure{
        \includegraphics[width=0.35\linewidth]{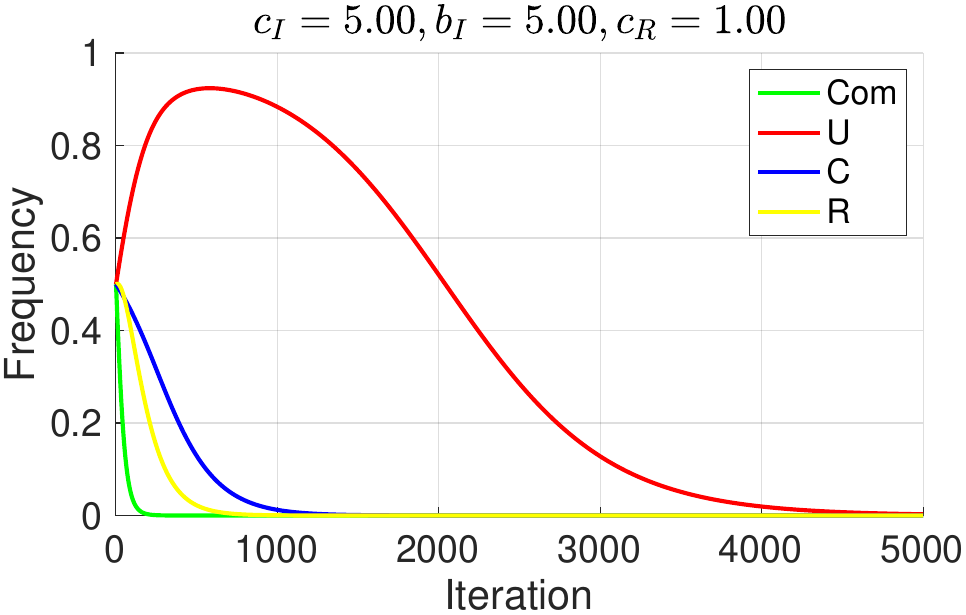}
    }
    \caption{Numerical integration of the evolution equation for two models, with negative impact of unsafe AI ($\epsilon = -0.1$) and high investigation cost ($c_I$=5.0). The left column shows the results of Model I, and the right column shows the results of Model II. Parameters are set as  $b_U = 4, b_P = 4, b_R = 4, c_P = 0.5, c_w = 1, u = 1.5, v = 0.5, b_{f_o} = 1,  p_w = 0.5$.}
    \label{fig:NumericalIntegration4}
\end{figure}

\begin{figure}
    \centering
    \subfigure{\label{fig:barcharta}
     \includegraphics[width=0.3\linewidth]{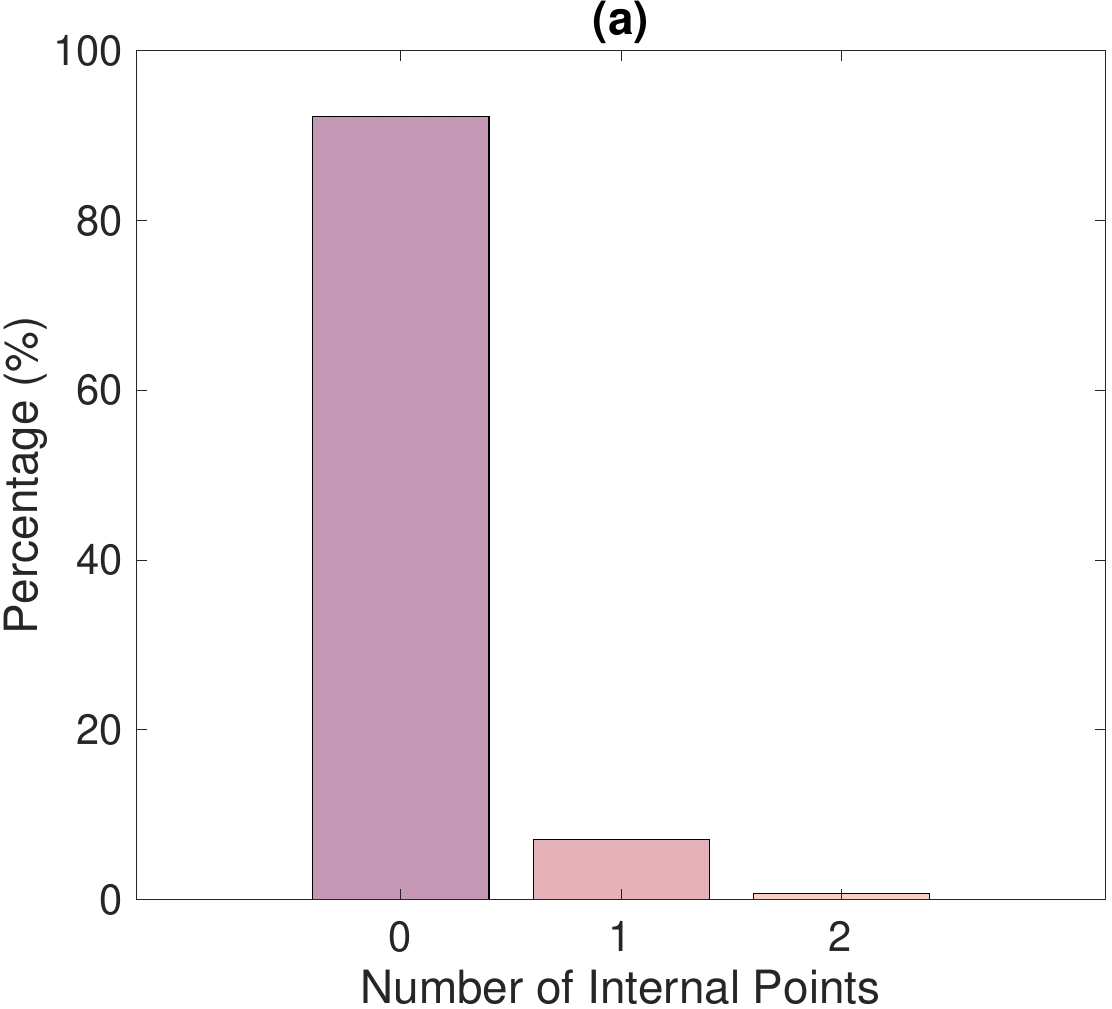} 
    }
    \subfigure{\label{fig:barchartb}
 \includegraphics[width=0.3\linewidth]{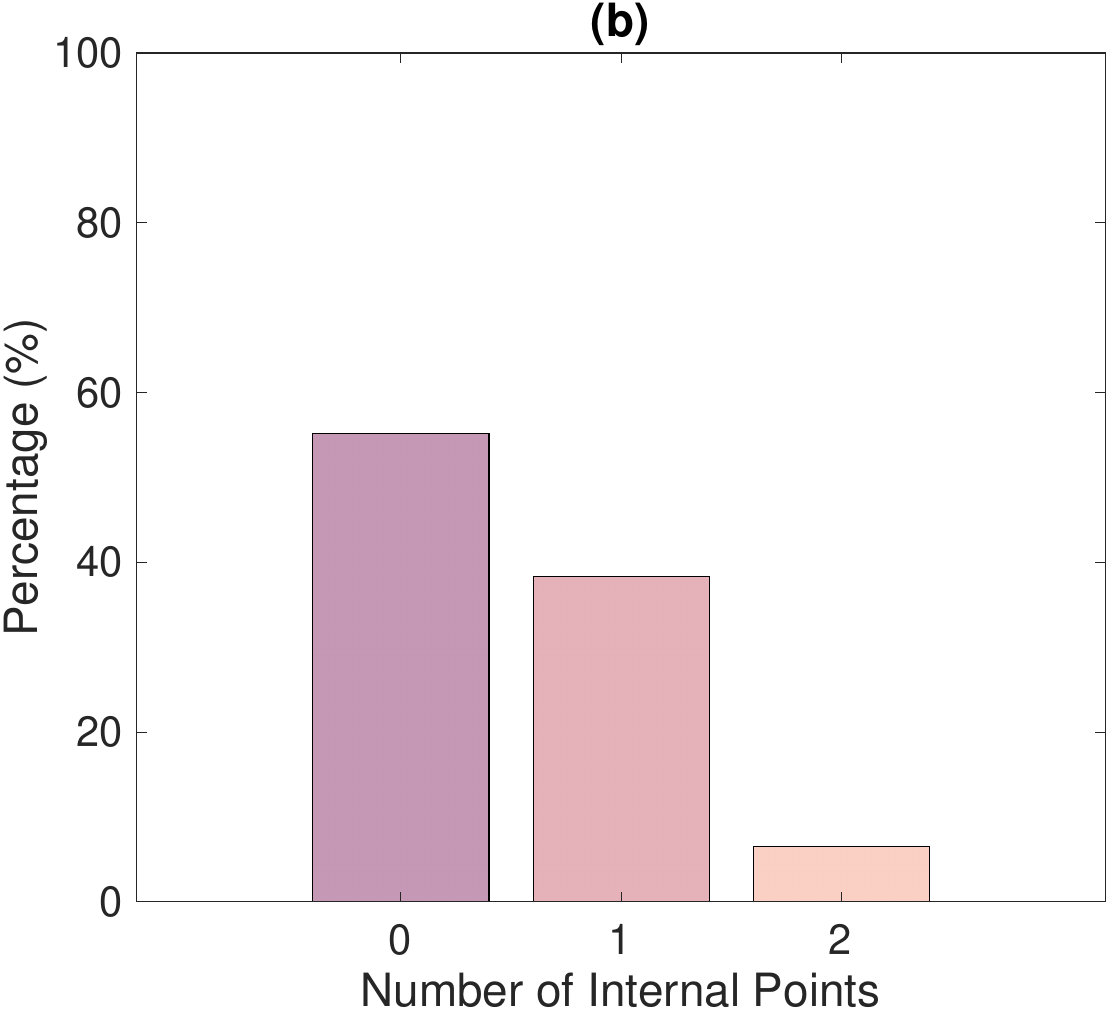} 
    }
    \caption{Percentage of the number of internal fixed points for two models. Shown are the results of 10,000 independent calculations. Parameters a randomly selected $b_I\in[0,4], b_U\in[0,4], b_P\in[0,4], b_R\in[0,4], c_I\in[0,1], c_P\in[0,1], c_R\in[0,1], c_w \in[0,1], v\in[0,1], u\in[0,5v], b_{f_o}\in[0,5], \epsilon\in[-2,1], p_w\in[0,1], c_w\in[0,1]$.}
\end{figure}

\clearpage

\bibliographystyle{IEEEtran}
\bibliography{refs} 

\begin{thebibliography}{10}
\providecommand{\url}[1]{#1}
\csname url@samestyle\endcsname
\providecommand{\newblock}{\relax}
\providecommand{\bibinfo}[2]{#2}
\providecommand{\BIBentrySTDinterwordspacing}{\spaceskip=0pt\relax}
\providecommand{\BIBentryALTinterwordstretchfactor}{4}
\providecommand{\BIBentryALTinterwordspacing}{\spaceskip=\fontdimen2\font plus
\BIBentryALTinterwordstretchfactor\fontdimen3\font minus \fontdimen4\font\relax}
\providecommand{\BIBforeignlanguage}[2]{{%
\expandafter\ifx\csname l@#1\endcsname\relax
\typeout{** WARNING: IEEEtran.bst: No hyphenation pattern has been}%
\typeout{** loaded for the language `#1'. Using the pattern for}%
\typeout{** the default language instead.}%
\else
\language=\csname l@#1\endcsname
\fi
#2}}
\providecommand{\BIBdecl}{\relax}
\BIBdecl

\bibitem{powers2023stuff}
S.~T. Powers, O.~Linnyk \emph{et~al.}, ``{The Stuff We Swim in: Regulation Alone Will Not Lead to Justifiable Trust in AI},'' \emph{IEEE Technology and Society Magazine}, vol.~42, no.~4, pp. 95--106, 2023.

\bibitem{clark2019regulatory}
J.~Clark and G.~K. Hadfield, ``{Regulatory Markets for AI Safety},'' \emph{arXiv}, Dec. 2019.

\bibitem{anderljung2023frontier}
M.~Anderljung, J.~Barnhart \emph{et~al.}, ``{Frontier AI Regulation: Managing Emerging Risks to Public Safety},'' Jul. 2023.

\bibitem{bengio2024managing}
\BIBentryALTinterwordspacing
Y.~Bengio, G.~Hinton, A.~Yao, D.~Song, P.~Abbeel \emph{et~al.}, ``Managing extreme {{AI}} risks amid rapid progress,'' \emph{Science}, vol. 384, no. 6698, pp. 842--845, May 2024. [Online]. Available: \url{https://www.science.org/doi/abs/10.1126/science.adn0117}
\BIBentrySTDinterwordspacing

\bibitem{hammond2025multiagentrisksadvancedai}
\BIBentryALTinterwordspacing
L.~Hammond, A.~Chan, J.~Clifton, J.~Hoelscher-Obermaier, A.~Khan, E.~McLean, C.~Smith, W.~Barfuss, J.~Foerster, T.~Gavenčiak, T.~A. Han, E.~Hughes, V.~Kovařík, J.~Kulveit, J.~Z. Leibo, C.~Oesterheld, C.~S. de~Witt, N.~Shah, M.~Wellman, P.~Bova, T.~Cimpeanu, C.~Ezell, Q.~Feuillade-Montixi, M.~Franklin, E.~Kran, I.~Krawczuk, M.~Lamparth, N.~Lauffer, A.~Meinke, S.~Motwani, A.~Reuel, V.~Conitzer, M.~Dennis, I.~Gabriel, A.~Gleave, G.~Hadfield, N.~Haghtalab, A.~Kasirzadeh, S.~Krier, K.~Larson, J.~Lehman, D.~C. Parkes, G.~Piliouras, and I.~Rahwan, ``Multi-agent risks from advanced ai,'' 2025. [Online]. Available: \url{https://arxiv.org/abs/2502.14143}
\BIBentrySTDinterwordspacing

\bibitem{hadfield2023regulatory}
G.~K. Hadfield and J.~Clark, ``{Regulatory Markets: The Future of AI Governance},'' Apr. 2023.

\bibitem{lewis2022like}
P.~R. Lewis and S.~Marsh, ``What is it like to trust a rock? a functionalist perspective on trust and trustworthiness in artificial intelligence,'' \emph{Cognitive Systems Research}, vol.~72, pp. 33--49, 2022.

\bibitem{sutrop2019should}
M.~Sutrop, ``Should we trust artificial intelligence?'' \emph{Trames}, vol.~23, no.~4, pp. 499--522, 2019.

\bibitem{lansing2016trust}
J.~Lansing and A.~Sunyaev, ``Trust in cloud computing: Conceptual typology and trust-building antecedents,'' \emph{ACM sigmis database: The database for advances in Information Systems}, vol.~47, no.~2, pp. 58--96, 2016.

\bibitem{andras2018trusting}
P.~Andras, L.~Esterle, M.~Guckert, T.~A. Han, P.~R. Lewis, K.~Milanovic, T.~Payne, C.~Perret, J.~Pitt, S.~T. Powers \emph{et~al.}, ``Trusting intelligent machines: Deepening trust within socio-technical systems,'' \emph{IEEE Technology and Society Magazine}, vol.~37, no.~4, pp. 76--83, 2018.

\bibitem{sigmund2010calculus}
K.~Sigmund, ``The calculus of selfishness,'' in \emph{The Calculus of Selfishness}.\hskip 1em plus 0.5em minus 0.4em\relax Princeton University Press, 2010.

\bibitem{hofbauer1998evolutionary}
J.~Hofbauer and K.~Sigmund, \emph{Evolutionary games and population dynamics}.\hskip 1em plus 0.5em minus 0.4em\relax Cambridge university press, 1998.

\bibitem{binmore2005natural}
K.~Binmore, \emph{Natural justice}.\hskip 1em plus 0.5em minus 0.4em\relax Oxford University Press, 2005.

\bibitem{han2020regulate}
T.~A. Han, L.~M. Pereira \emph{et~al.}, ``{ To Regulate or Not: A Social Dynamics Analysis of an Idealised AI Race },'' \emph{Journal of Artificial Intelligence Research}, vol.~69, pp. 881--921, Nov. 2020.

\bibitem{han2022voluntary}
T.~A. Han, T.~Lenaerts \emph{et~al.}, ``{ Voluntary Safety Commitments Provide an Escape from Over-Regulation in AI Development },'' \emph{Technology in Society}, vol.~68, p. 101843, 2022.

\bibitem{bova2023both}
P.~Bova, A.~Di~Stefano, and T.~A. Han, ``Both eyes open: Vigilant incentives help auditors improve ai safety,'' \emph{Journal of Physics: Complexity}, vol.~5, no.~2, p. 025009, 2024.

\bibitem{alalawi2024trust}
Z.~Alalawi, P.~Bova, T.~Cimpeanu, A.~Di~Stefano, M.~H. Duong, E.~F. Domingos, T.~A. Han, M.~Krellner, B.~Ogbo, S.~T. Powers \emph{et~al.}, ``Trust ai regulation? discerning users are vital to build trust and effective ai regulation,'' \emph{arXiv preprint arXiv:2403.09510}, 2024.

\bibitem{yang2023ai}
S.~Yang, N.~M. Krause, L.~Bao, M.~N. Calice, T.~P. Newman, D.~A. Scheufele, M.~A. Xenos, and D.~Brossard, ``In ai we trust: The interplay of media use, political ideology, and trust in shaping emerging ai attitudes,'' \emph{Journalism \& Mass Communication Quarterly}, p. 10776990231190868, 2023.

\bibitem{maggetti_media_2012}
\BIBentryALTinterwordspacing
M.~Maggetti, ``\BIBforeignlanguage{en}{The media accountability of independent regulatory agencies},'' \emph{\BIBforeignlanguage{en}{European Political Science Review}}, vol.~4, no.~3, pp. 385--408, Nov. 2012. [Online]. Available: \url{https://www.cambridge.org/core/journals/european-political-science-review/article/media-accountability-of-independent-regulatory-agencies/DF8416832F2BD5197D6C723BE55DB5DF}
\BIBentrySTDinterwordspacing

\bibitem{mccombs1972agenda}
M.~E. McCombs and D.~L. Shaw, ``The agenda-setting function of mass media,'' \emph{Public Opinion Quarterly}, vol.~36, no.~2, pp. 176--187, 1972.

\bibitem{zhao2023survey}
W.~X. Zhao, K.~Zhou, J.~Li, T.~Tang, X.~Wang, Y.~Hou, Y.~Min, B.~Zhang, J.~Zhang, Z.~Dong \emph{et~al.}, ``A survey of large language models,'' \emph{arXiv preprint arXiv:2303.18223}, vol.~1, no.~2, 2023.

\bibitem{lu2024llms}
Y.~Lu, A.~Aleta, C.~Du, L.~Shi, and Y.~Moreno, ``Llms and generative agent-based models for complex systems research,'' \emph{Physics of Life Reviews}, 2024.

\bibitem{park2023generative}
J.~S. Park, J.~O'Brien, C.~J. Cai, M.~R. Morris, P.~Liang, and M.~S. Bernstein, ``Generative agents: Interactive simulacra of human behavior,'' in \emph{Proceedings of the 36th annual acm symposium on user interface software and technology}, 2023, pp. 1--22.

\bibitem{bail2024can}
C.~A. Bail, ``Can generative ai improve social science?'' \emph{Proceedings of the National Academy of Sciences}, vol. 121, no.~21, p. e2314021121, 2024.

\bibitem{buscemi2024large}
A.~Buscemi and D.~Proverbio, ``Large language models' detection of political orientation in newspapers,'' \emph{arXiv preprint arXiv:2406.00018}, 2024.

\bibitem{buscemi2024chatgpt}
------, ``Chatgpt vs gemini vs llama on multilingual sentiment analysis,'' \emph{arXiv preprint arXiv:2402.01715}, 2024.

\bibitem{lee2024evaluating}
N.~Lee, J.~Hong, and J.~Thorne, ``Evaluating the consistency of llm evaluators,'' \emph{arXiv preprint arXiv:2412.00543}, 2024.

\bibitem{chatgpt}
\BIBentryALTinterwordspacing
OpenAI. (2023) Introducing chatgpt. [Online]. Available: \url{https://openai.com/blog/chatgpt}
\BIBentrySTDinterwordspacing

\bibitem{mistral}
\BIBentryALTinterwordspacing
M.~AI. (2025) Au large. [Online]. Available: \url{https://mistral.ai/news/mistral-large}
\BIBentrySTDinterwordspacing

\bibitem{buscemi2025fairgame}
A.~Buscemi, D.~Proverbio, A.~Di~Stefano, T.~A. Han, and P.~Liò, ``Fairgame: a framework for ai agents bias recognition using game theory,'' \emph{in preparation}, 2025.

\bibitem{traulsen2006}
A.~Traulsen, M.~A. Nowak, and J.~M. Pacheco, ``Stochastic dynamics of invasion and fixation,'' \emph{Phys. Rev. E}, vol.~74, p. 11909, 2006.

\bibitem{key:imhof2005}
L.~A. Imhof, D.~Fudenberg, and M.~A. Nowak, ``Evolutionary cycles of cooperation and defection,'' \emph{Proc. Natl. Acad. Sci. U.S.A.}, vol. 102, pp. 10\,797--10\,800, 2005.

\bibitem{key:novaknature2004}
M.~A. Nowak, A.~Sasaki, C.~Taylor, and D.~Fudenberg, ``Emergence of cooperation and evolutionary stability in finite populations,'' \emph{Nature}, vol. 428, pp. 646--650, 2004.

\bibitem{domingos2023egttools}
E.~F. Domingos, F.~C. Santos, and T.~Lenaerts, ``Egttools: Evolutionary game dynamics in python,'' \emph{Iscience}, vol.~26, no.~4, 2023.

\bibitem{encarnaccao2016paradigm}
S.~Encarna{\c{c}}{\~a}o, F.~P. Santos, F.~C. Santos, V.~Blass, J.~M. Pacheco, and J.~Portugali, ``Paradigm shifts and the interplay between state, business and civil sectors,'' \emph{Royal Society open science}, vol.~3, no.~12, p. 160753, 2016.

\bibitem{alalawi2019pathways}
Z.~Alalawi, T.~A. Han, Y.~Zeng, and A.~Elragig, ``Pathways to good healthcare services and patient satisfaction: An evolutionary game theoretical approach,'' in \emph{Artificial Life Conference Proceedings}.\hskip 1em plus 0.5em minus 0.4em\relax MIT Press One Rogers Street, Cambridge, MA 02142-1209, USA journals-info~…, 2019, pp. 135--142.

\bibitem{taylor1979evolutionarily}
P.~D. Taylor, ``Evolutionarily stable strategies with two types of player,'' \emph{Journal of applied probability}, vol.~16, no.~1, pp. 76--83, 1979.

\bibitem{bauer2019stabilization}
J.~Bauer, M.~Broom, and E.~Alonso, ``The stabilization of equilibria in evolutionary game dynamics through mutation: mutation limits in evolutionary games,'' \emph{Proceedings of the Royal Society A}, vol. 475, no. 2231, p. 20190355, 2019.

\bibitem{nowak2004emergence}
M.~A. Nowak, A.~Sasaki, C.~Taylor, and D.~Fudenberg, ``Emergence of cooperation and evolutionary stability in finite populations,'' \emph{Nature}, vol. 428, no. 6983, pp. 646--650, 2004.

\bibitem{rand2013evolution}
D.~G. Rand, C.~E. Tarnita, H.~Ohtsuki, and M.~A. Nowak, ``Evolution of fairness in the one-shot anonymous ultimatum game,'' \emph{Proceedings of the National Academy of Sciences}, vol. 110, no.~7, pp. 2581--2586, 2013.

\bibitem{zisis2015generosity}
I.~Zisis, S.~Di~Guida, T.~A. Han, G.~Kirchsteiger, and T.~Lenaerts, ``Generosity motivated by acceptance-evolutionary analysis of an anticipation game,'' \emph{Scientific reports}, vol.~5, no.~1, p. 18076, 2015.

\end{thebibliography}
\end{document}